\documentclass[journal,12pt,onecolumn]{IEEEtran}

\IEEEoverridecommandlockouts                              % This command is only needed if 
\usepackage{mathptmx} % assumes new font selection scheme installed
\usepackage{amsmath} % assumes amsmath package installed
\usepackage{amssymb}  % assumes amsmath package installed
\usepackage{verbatim}
\usepackage{url}
 \usepackage[table,xcdraw]{xcolor}

\usepackage{longtable}
\usepackage{lscape}
\usepackage{cite}
\usepackage{graphicx}
\usepackage{multicol}

\usepackage{mathtools}
\mathtoolsset{showonlyrefs}

\usepackage{todonotes}

\usepackage{algpseudocode}

\usepackage[caption=false,font=footnotesize]{subfig}
\usepackage{enumerate}   
\usepackage{hyperref}
\hypersetup{
    colorlinks=true, %set true if you want colored links
    linktoc=all,     %set to all if you want both sections and subsections linked
    menucolor=black,
    linkcolor =black,
    anchorcolor =black,
    citecolor =black,
    urlcolor= black,
}

\usepackage{color}
\usepackage{pgfplots}
\pgfplotsset{compat=1.5}
\usepgfplotslibrary{statistics}
\pgfplotsset{grid style={dashed,gray}}
\usepackage[linesnumbered,ruled]{algorithm2e}
\usepackage[ampersand]{easylist}
\usepackage{footnote}
\makesavenoteenv{tabular}
\makesavenoteenv{table}
\usepackage[utf8]{inputenc}
\usepackage[linesnumbered,ruled]{algorithm2e}

\def\Red#1{\textcolor{red}{#1}}

\usepackage{booktabs}
\usepackage[]{algorithm2e}

\newtheorem{theorem}{Theorem}

\newtheorem{problem}[theorem]{Problem}
\usepackage{float} 

\newcommand{\timme}{t}
\newcommand{\disc}{k}

\newcommand{\volume}{\mathcal{V}}

\newcommand{\occupiedspace}{\volume_{\mathrm{occ}}}

\newcommand{\occupiedspaceMem}{\volume^*_{\mathrm{occ}}}
\newcommand{\freespace}{\volume^*_{\mathrm{free}}}

\newcommand{\navigationspace}{\volume_{\mathrm{nav}} }
\newcommand{\castspace}{\volume^*_{\mathrm{cst}}}
\newcommand{\pocc}{P_{\text{occ}}}
\newcommand{\pray}{P_{\text{ray}}}

\newcommand{\state}{z}

\newcommand{\start}{i}
\newcommand{\final}{f}
\newcommand{\current}{c}

\newcommand{\Ball}{\mathcal{B}}
\newcommand{\position}{p}
\newcommand{\velocity}{v}
\newcommand{\acceleration}{a}
\newcommand{\jerk}{j}
\newcommand{\snap}{\zeta}

\newcommand{\saferadious}{R}

\newcommand{\yaw}{\psi}
\newcommand{\angvelocity}{\omega}
\newcommand{\angacceleration}{\alpha}

\newcommand{\frustum}{\mathcal{F}}
\newcommand{\framerate}{f_{\text{cam}}}
\newcommand{\maplevel}{l}
\newcommand{\voxelrad}{r}
\newcommand{\maxlevel}{L}
\newcommand{\console}[1]{\texttt{\small #1}}
\newcommand{\voxel}{v}
\newcommand{\waypoints}{W}

\newcommand{\Tree}{T}

\newcommand{\omegamax}{\Omega_{\mbox{\scriptsize max}}}

\newcommand{\vmax}{V_{\mbox{\scriptsize max}}}
\newcommand{\amax}{A_{\mbox{\scriptsize max}}}
\newcommand{\jmax}{J_{\mbox{\scriptsize max}}}

\newcommand{\ellipse}{\mathcal{E}}

\newcommand{\sallowed}{\mathcal{Z}_\text{safe}}

\title{
Autonomous Navigation of MAVs in Unknown Cluttered Environments
}

 \author{Leobardo~Campos-Macías, Rodrigo~Aldana-López, Rafael~de~la~Guardia, José~I.~Parra-Vilchis and David~Gómez-Gutiérrez \footnote{\Red{This is the preprint version of the accepted Manuscript: L.~Campos-Macías, R.~Aldana-López, R.~de~la~Guardia, J.~I.~Parra-Vilchis and D.~Gómez-Gutiérrez , ``Autonomous Navigation of MAVs in Unknown Cluttered Environments”, Journal of Field Robotics, 2020, ISSN: 1556-4967. DOI. 10.1002/rob.21959.
Please cite the publisher's version. For the publisher's version and full citation details see:
\url{https://doi.org/10.1002/rob.21959}.
}}

 \thanks{The authors are with the Multi-Agent Autonomous Systems Lab, Intel Labs, Intel Tecnología de México. L.~Campos-Macías is also with Cinvestav, Unidad Guadalajara. D.~Gómez-Gutiérrez is also with Tecnologico de Monterrey, Escuela de Ingeniería y Ciencias. (e-mail:lecampos@gdl.cinvestav.mx; rodrigo.aldana.lopez@gmail.com; rafael.de.la.guardia@intel.com; jose.i.parra.vilchis@intel.com, david.gomez.g@ieee.org).}
\thanks{Corresponding Author: D.~Gómez-Gutiérrez}
}

\begin{document}

\maketitle

\overrideIEEEmargins
%%%%%%%%%%%%%%%%%%%%%%%%%%%%%%%%%%%%%%%%%%%%%%%%%%%%%%%%%%%%%%%%%%%%%%%%%%%%%%%%
\begin{abstract}
This paper presents an autonomous navigation framework for reaching a goal in unknown 3D cluttered environments. The framework consists of three main components. First, a computationally efficient method for mapping the environment from the disparity measurements obtained from a depth sensor. Second, a stochastic method to generate a path to a given goal, taking into account field of view constraints on the space that is assumed to be safe for navigation. Third, a fast method for the online generation of motion plans, taking into account the robot's dynamic constraints, model and environmental uncertainty and disturbances. To highlight the contribution with respect to the available literature, we provide a qualitative and quantitative comparison with state of the art methods for reaching a goal and for exploration in unknown environments, showing the superior performance of our approach. To illustrate the effectiveness of the proposed framework, we present experiments in multiple indoors and outdoors environments running the algorithm fully on board and in real-time, using a robotic platform based on the Intel Ready to Fly drone kit, which represents the implementation in the most frugal platform for navigation in unknown cluttered environments demonstrated to date. Open source code is available at:~\url{https://github.com/IntelLabs/autonomousmavs}. The video of the experimental results can be found at~\url{https://youtu.be/Wq0e7vF6nZM}

\end{abstract}

%%%%%%%%%%%%%%%%%%%%%%%%%%%%%%%%%%%%%%%%%%%%%%%%%%%%%%%%%%%%%%%%%%%%%%%%%%%%%%%%
\section{Introduction}
Autonomous navigation in unknown cluttered environments is one of the fundamental problems in robotics with applications in search and rescue, information gathering and inspection of industrial and civil structures, among others. Multirotor Micro Aerial Vehicles (MAVs) are an ideal robotic platform for many of these tasks due to their small size and high maneuverability. MAVs have also been demonstrated to be powerful enough platforms to operate autonomously in GPS-denied environments using only onboard processing and sensing. In many practical use cases, navigation is goal-oriented, meaning that there is a premium on reaching a specific target location as directly as possible, for example, while minimizing time, energy or total distance traveled by the MAV. 

Since the environment is unknown, building a map is an important part of the solution as it provides a world representation to the robot that may include metric, topological and semantic information needed by the robot to plan trajectories towards its global goal. Continuous advances in computing hardware and software, availability of powerful sensors, and algorithmic developments~\cite{Cadena2016PastAge}, enable maps to be processed fast enough to also be useful for collision avoidance with objects entering the field of view of the robot. As the map is created and new areas are discovered, online planning is required as it helps the robot identify shortcuts and escape local minima. Such planning algorithms must take into account the dynamical constraints of the robot to guarantee that the trajectories planned can be executed.

Although mapping, planning, and trajectory generation can be considered mature fields considering certain combinations of robotic platforms and environments, a framework combining elements from all these fields for MAV navigation in general environments is still missing. Recent surveys on classical approaches for mapping can be found in \cite{Cadena2016PastAge}, for planning in \cite{Tsardoulias2016ADensity} and for robotic navigation in general in \cite{Thrun2006ProbabilisticRobotics}.

 \begin{figure}[t]
     \centering
     \includegraphics[scale=0.16]{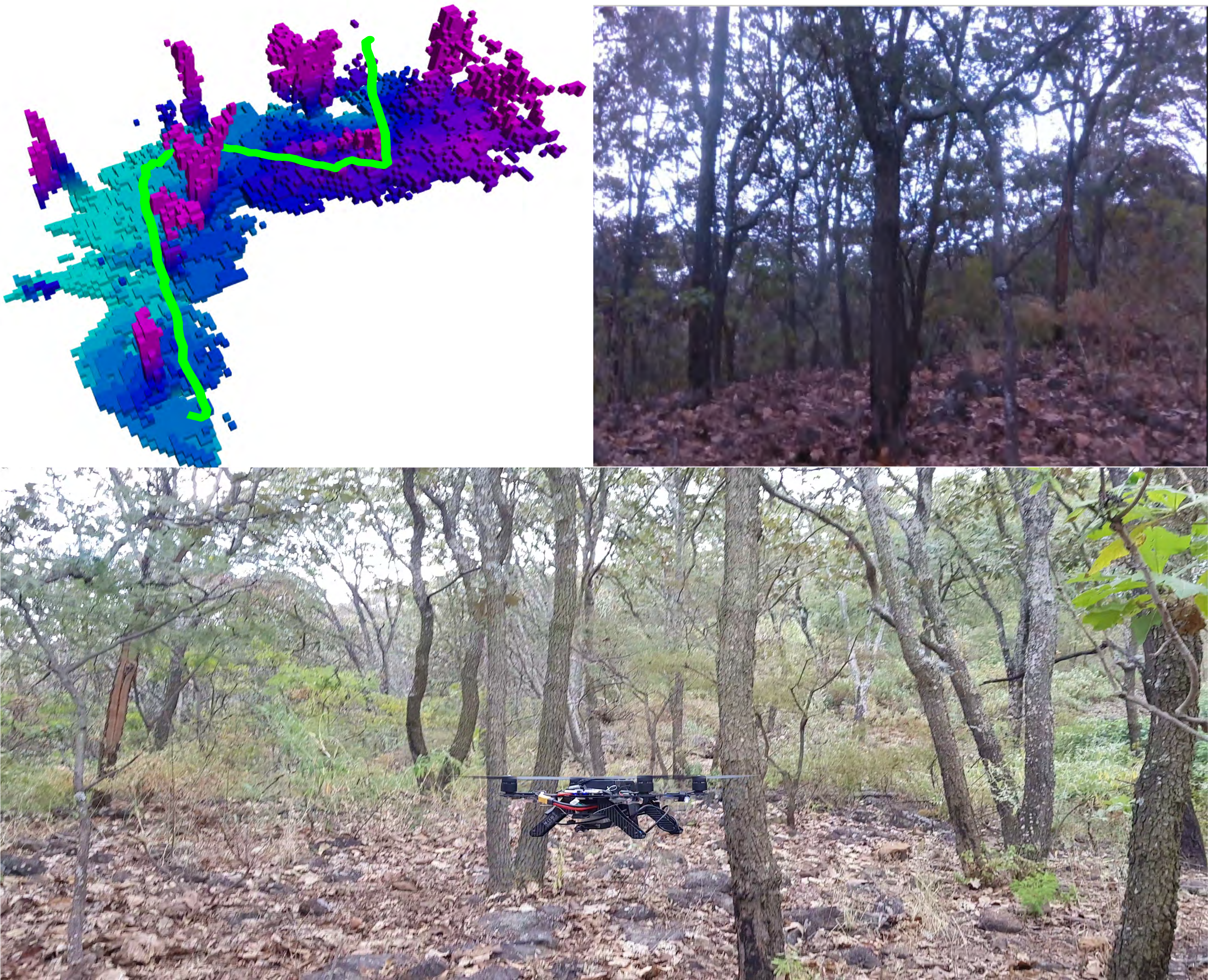}
     \caption{Reaching a goal in a forest environment: The top left image shows the map representation generated together with the generated trajectory. The top right image shows the camera robot footage. The bottom image shows a picture of the robot in the forest performing the experiment.}
     \label{fig:mainfig}
 \end{figure}

\subsection{Contributions}
In this paper, we present a framework for MAV navigation in complex, unknown 3D environments. This framework enables robust autonomous navigation in resource-constrained platforms that, using previous methods, would not be able to navigate in cluttered unknown environments. 

It consists of three main components. First, a computationally efficient method for mapping the environment from the disparity measurements obtained from an RGB-D camera. Second, a stochastic method to generate a path to a goal located on the map, taking into account field of view constraints on the space that is assumed to be safe for navigation. This method is efficient because it plans only in three-dimensional space and does not require taking into account high order dynamics of the robot which are handled in the motion planning method. In addition, in order to detect collisions, the path planning exploits the structure defined by our map representation for fast collision checking. Third, a lightweight method for the online generation of motion plans over a three-dimensional path, taking into account the MAV's dynamic constraints, model and environmental uncertainty, and disturbances. Contrary to existing approaches, this method does not require the use of any optimization solver, making it suitable for online planning, even in platforms with low computing power. In addition, the combination of these three main components allows our system to reach a global goal without getting stuck in dead ends. To demonstrate the efficiency of our method, experiments in multiple indoors and outdoors environments, like the ones shown in Figure~\ref{fig:mainfig}, were implemented in a MAV platform equipped with a low power microprocessor. It is important to highlight that existing implementations solving similar problems require
high-end general purpose and graphics processors and significantly higher compute resources.
%significantly higher compute with high-end general purpose and graphics processors.
%is the most resource-constrained platform compute platform in which autonomous navigation in general unknown environments has been demonstrated to date. It is important to highlight that, existing implementations solving similar problems require a high-end computer.

% In this contribution we present a complete framework for goal-oriented navigation in an unknown cluttered environment. It includes a computationally efficient method for mapping the environment from the disparity measurements obtained from an RGB-D camera, a stochastic method to plan safe local and global paths and a fast method for the online generation of safe motion plans, taking into account the MAV's dynamic constraints, model and environmental uncertainty and disturbances. Before entering into the details of our solution, we present a survey of related work in autonomous MAV navigation. Recent surveys on classical approaches for mapping, planning and for robotic navigation in general can be found in~\cite{Thrun2006ProbabilisticRobotics}.

\subsection{Related work}
% The papers ~\cite{Shen2011AutonomousMAV,Fraundorfer2012Vision-basedMAV} are among the earliest works demonstrating fully autonomous MAV navigation. ~\cite{Shen2011AutonomousMAV} presented a system that enables autonomous navigation for multi-floor mapping with all the processing onboard a MAV including loop closure, localization, planning, and control. The map was constructed using a simplified incremental SLAM algorithm assuming 2.5D environment models formed by vertical walls and horizontal ground planes. Their MAV was capable of following waypoints set manually in the currently estimated map using proportional-derivative feedback control laws to transform these kinematic inputs to appropriate dynamic control inputs.~\cite{Fraundorfer2012Vision-basedMAV} demonstrated the feasibility of autonomous mapping and exploration for a quadrotor MAV with a  forward-looking stereo camera as its main sensor. The MAV was capable of path planning and exploration, moving autonomously from one point to another using the VFH+ algorithm for collision avoidance. Global planning was done using frontier-based exploration applied on a 2D slice of the current occupancy grid map.

The papers~\cite{Shen2011AutonomousMAV,Fraundorfer2012Vision-basedMAV} are among the earliest works demonstrating fully autonomous MAV navigation. The authors in~\cite{Shen2011AutonomousMAV} presented a system that enables autonomous navigation for multi-floor mapping with all the processing onboard a MAV including loop closure, localization, planning, and control. However, the map was constructed using a simplified incremental SLAM algorithm assuming 2.5D environment models formed by vertical walls and horizontal ground planes. The work in~\cite{Fraundorfer2012Vision-basedMAV} demonstrated the feasibility of autonomous mapping and exploration for a quadrotor MAV with a  forward-looking stereo camera as its main sensor. The MAV was capable of path planning and exploration, moving autonomously from one point to another with collision avoidance. However, global planning was done using frontier-based exploration applied on a 2D slice of the current occupancy grid map.

More recent papers have continued the strategy of using refinements on the frontiers-based algorithm~\cite{Yamauchi1997AExploration} for global exploration, combined with some variation of the next best view algorithm~\cite{Connolly1985TheViews} for local planning. 
The paper~\cite{Cieslewski2017RapidFlight} proposes an exploration algorithm that is designed to fly at high velocities as much as possible. Instead of planning trajectories, a reactive mode generates instantaneous velocity commands based on currently observed frontiers. The desired velocity at a frontier is proportional to its distance from the MAV so that for a frontier at the detection range the velocity will be maximum and pointing towards the unknown volume. Unfortunately, the method does not guarantee alignment between the field of view of the camera and the direction of motion, which may hinder obstacle avoidance. In case that no frontiers are observed in the field of view their method falls back to classical frontier-based exploration using a regular 3D voxel grid representation of the environment. In contrast, in our framework we use a more efficient environment representation based on a linear octree.

The contribution in~\cite{Papachristos2017Uncertainty-awareRobots} aims to provide accurate mapping simultaneously with the exploration of unknown environments. Their method combines paths sampled from two random trees, where the first tree selects viewpoints to visit next based on the number of unexplored voxels and the second tree is used to find a path to the selected viewpoint that minimizes the robot’s pose and landmarks uncertainty. Paths are only planned inside free space such that they are collision free and can be tracked by the vehicle, given possible motion constraints. However, to guarantee that the paths are collision free they use expensive simulation/optimization-based methods.

In~\cite{Selin2019EfficientEnvironments}, receding horizon Next-Best-View (NBV) planning and frontier exploration are combined for local and global exploration, respectively. The potential information gain for every yaw angle is estimated using sparse raycasting to select the next direction to explore in NBV. Nevertheless, this method does not have a goal oriented objective. The work in~\cite{Oleynikova2018AEnvironments} describes a system for mapping and planning based on incrementally built signed distance fields in dynamically growing maps. An intermediate goal-finding algorithm complements a conservative local planner which treats unknown space as occupied and inaccessible. The method uses an incrementally-built, dynamically-growing Truncated Signed Distance Field (TSDF) map representation to compute collision costs and gradients. The next intermediate goal for navigation is selected from a set of candidate points sampled from the unoccupied space of the TSDF around the robot by solving an optimization problem to maximize a reward combining an exploration gain and the distance to the global goal. Compared to our proposed method, computing signed distance fields and the need to solve optimization problems online makes their approach computationally more expensive.

The papers~\cite{Lin2018AutonomousFusion} and~\cite{Mohta2018FastEnvironments} focused on navigation at high speeds to a goal in fully unknown and cluttered 3D environments. In~\cite{Lin2018AutonomousFusion} the authors used a local planner to find a path to a target on a global guiding path. The local map is built from depth images using GPU-accelerated TSDF fusion. The perception-action loop is closed using a standard multirotor controller to execute the trajectory generated by a nonlinear optimization solver, which is initialized with the path computed by the local planner. The overall method is compute intensive, requiring both a high-end CPU and a GPU. They report flying in a sparsely cluttered environment at up to 2.2 m/s but with an average whole system delay of 230 ms navigation at such high velocities would not be safe in general unknown environments. In~\cite{Mohta2018FastEnvironments} the authors used a local robot-centric map and a local planner to generate the robot trajectory. To escape dead ends and obtain globally consistent local actions an A* planner is provided with a hybrid map formed by combining a local 3D map with a 2D global map. The piecewise linear path from the planner guides the convex decomposition of the local map to find a safe corridor in free space formed as linear equality constraints in a quadratic program (QP) for collision checking. The QP also takes into account dynamic constraints and a modified cost functional such that the generated trajectories are close to the center of the safe corridor. According to the authors, one of the limiting factors for reaching high speeds in their real-world experiments was the size of the map for planning. The other key factor was the sparsity of the map which in their case was determined by the use of a nodding 2D lidar which could reliably run at only around 1.5 Hz. In contrast, our complete system runs at 30 Hz with an average delay of around 3.4 ms.

Among the recent methods focusing on navigation in cluttered environments, not necessarily on exploration, are the works of \cite{Usenko2017Real-TimeBuffer,Gao2018OnlinePolynomial,Florence2018NanoMap:Data}. The paper~\cite{Usenko2017Real-TimeBuffer} proposes a robocentric, fixed-size three-dimensional circular buffer to maintain local information about the environment, trading off the ability to arbitrarily large occupancy maps for faster lookup and measurement insertion operations. Points are inserted into the occupancy buffer by using raycast operations and voxels are updated by using the hit and miss probabilities similar to Octomap. The replanning problem is represented as an optimization of an endpoint cost that penalizes position and velocity deviations at the end of the trajectory, a collision cost function, the cost of the integral over the squared derivatives (acceleration, jerk, snap) and a soft limit on the norm of time derivatives (velocity, acceleration, jerk and snap) over the trajectory. However, this method may get stuck in local minima and the optimization does not take into account the yaw dynamics.
%\todo[inline]{agregar el drawback de GAO 17}
The authors in~\cite{Gao2018OnlinePolynomial} consider the navigation problem in unknown cluttered environments by a MAV equipped with a monocular fish-eye camera. A time-indexed path is generated by fast marching on a velocity field induced from an Euclidead Signed Distance Field (ESDF). The nodes from the path are then used for building a flight corridor in a voxel grid. Trajectory generation is done by posing a constrained quadratic optimization problem and solving it for the coefficients of a Bernstein polynomial which minimize the jerk along the trajectory for each of the three spatial dimensions. The constraints ensure the smoothness, safety as well as the dynamical feasibility of the trajectory. However this method is computationally expensive and does not guarantee that a feasible solution will be found. 
%A more recent work from the previous group~\cite{Zhou2019RobustFlight} proposes a kinodynamic path searching method based on heuristic search and linear quadratic minimum-time control. To keep processing time, low trajectories are generated only within a sphere with a radius of 5 m centered on the quadrotor. Path searching using motion primitives is used followed by QP optimization and time adjustment. For testing fast-replanning in aggressive flight a motion capture system was used and the map of the environment was pre-built. 
The work in~\cite{Florence2018NanoMap:Data} propose uncertainty-aware proximity queries for planning without any prior discretization of the data in a world frame. The method works by reverse searching over time through sensor measurement views until finding a satisfactory view of a subset of space. The pose uncertainty associated with depth sensor measurements is incorporated into planning by treating each pose with frame-specific uncertainty relative to the current body frame. Nonetheless, this method may get stucked in dead ends.

\section{Problem Statement and Overview of the Solution}

In this work, we address the problem of autonomous navigation of a holonomic robot in a 3D space, from an initial pose to a desired feasible pose in a cluttered unknown environment. It is considered that the robot has an odometry module %with a small drift, for instance, by using internal correction algorithms, such as bundle adjustment~\cite{Triggs2000BundleSynthesis}. This module provides 
to measure position and orientation with respect to the robot's initial pose. 
Existing visual inertial odometry techniques, such as~\cite{Qin2018VINS-Mono:Estimator}~\cite{Forster2016SVO:Systems}, are adequate for navigation as drift can be expected to be low at least within the local planning area around the robot.
In addition, we assume that the robot is equipped with at least one 3D sensor, for instance RGB-D or stereo cameras with finite resolution and known horizontal and vertical field of view.

\subsection{Notation}
\begin{figure}
    \centering
    %\includegraphics[scale=0.7]{images/path_diagram.png}
    %\vspace{-0.7in}
    \def\svgwidth{8.5cm}
\begingroup%
  \makeatletter%
  \providecommand\color[2][]{%
    \errmessage{(Inkscape) Color is used for the text in Inkscape, but the package 'color.sty' is not loaded}%
    \renewcommand\color[2][]{}%
  }%
  \providecommand\transparent[1]{%
    \errmessage{(Inkscape) Transparency is used (non-zero) for the text in Inkscape, but the package 'transparent.sty' is not loaded}%
    \renewcommand\transparent[1]{}%
  }%
  \providecommand\rotatebox[2]{#2}%
  \ifx\svgwidth\undefined%
    \setlength{\unitlength}{976.65446604bp}%
    \ifx\svgscale\undefined%
      \relax%
    \else%
      \setlength{\unitlength}{\unitlength * \real{\svgscale}}%
    \fi%
  \else%
    \setlength{\unitlength}{\svgwidth}%
  \fi%
  \global\let\svgwidth\undefined%
  \global\let\svgscale\undefined%
  \makeatother%
  \begin{picture}(1,0.71408012)%  
      \put(0,0){\includegraphics[width=\unitlength,page=1]{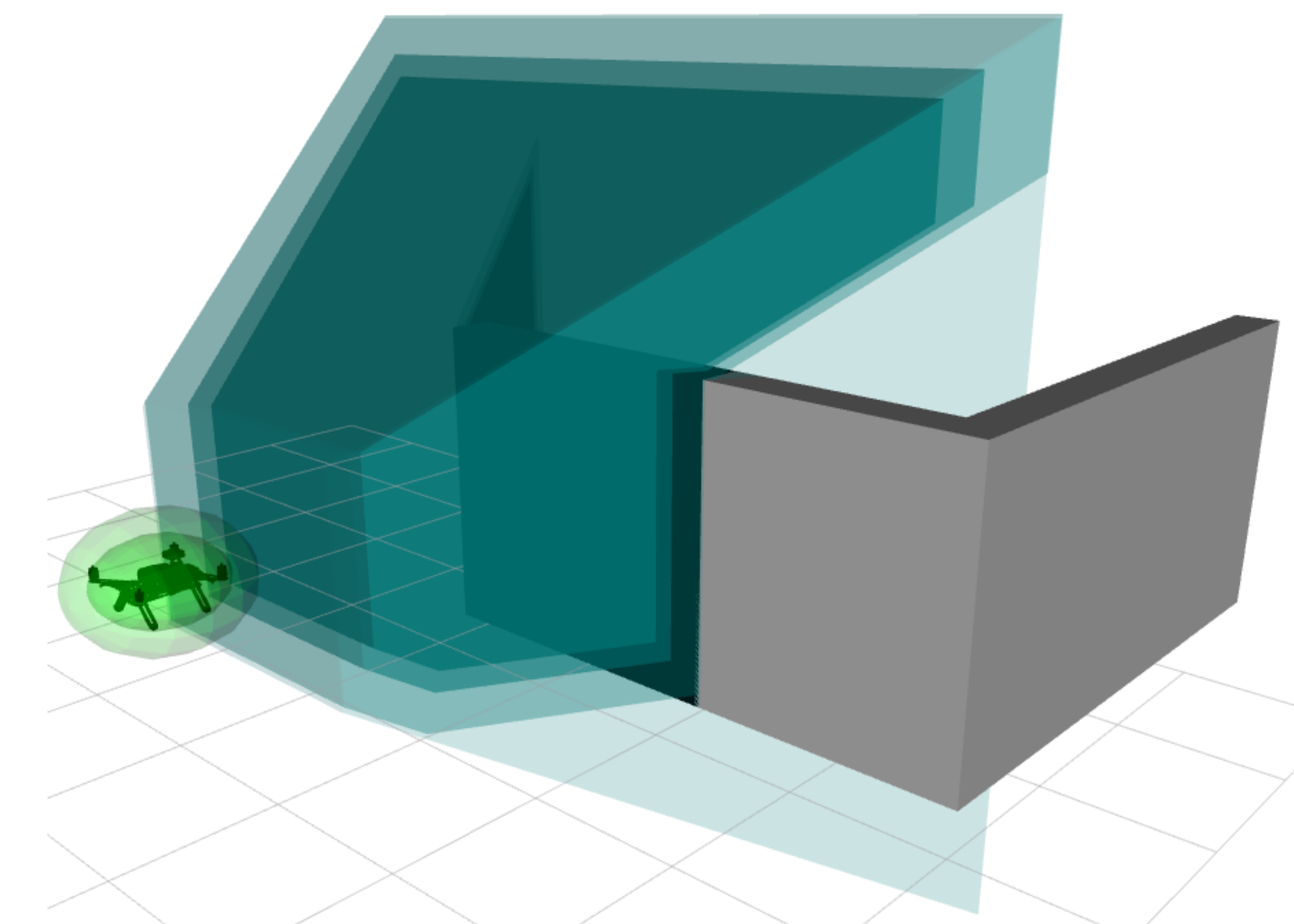}}%
      \put(0,0){\includegraphics[width=\unitlength,page=2]{images/regions.pdf}}%
      \put(0,0){\includegraphics[width=\unitlength,page=3]{images/regions.pdf}}%
      \put(0,0){\includegraphics[width=\unitlength,page=4]{images/regions.pdf}}%
      
      \footnotesize{
                \put(0.88980479,0.53113128){\color[rgb]{0,0,0}\makebox(0,0)[lb]{\smash{$\frustum[\disc]$}}}%
                \put(0.88980479,0.58222676){\color[rgb]{0,0,0}\makebox(0,0)[lb]{\smash{$\freespace[\disc]$}}}%
                \put(0.88980479,0.63332226){\color[rgb]{0,0,0}\makebox(0,0)[lb]{\smash{$\mathcal{V}_{\text{safe}}[\disc]$}}}%
                \put(0.88980479,0.68589955){\color[rgb]{0,0,0}\makebox(0,0)[lb]{\smash{$\castspace[\disc]$}}}%
                \put(0.91732415,0.12383893){\color[rgb]{0,0,0}\makebox(0,0)[lb]{\smash{$\occupiedspace$}}}%
                \put(0.51714966,0.08381387){\color[rgb]{0,0,0}\makebox(0,0)[lb]{\smash{$\occupiedspaceMem[\disc]$}}}%
                \put(-0.0028942,0.34526695){\color[rgb]{0,0,0}\makebox(0,0)[lb]{\smash{$\Ball_{\text{R}_\text{r}}$}}}%
                \put(0.03452444,0.13945469){\color[rgb]{0,0,0}\makebox(0,0)[lb]{\smash{$\Ball_\saferadious$}}}%
                
                \put(0.60163952,0.29631973){\color[rgb]{1,1,1}\rotatebox{-54.82633396}{\makebox(0,0)[lb]{\smash{WALL}}}}%
                \put(0.86170171,0.24572588){\color[rgb]{1,1,1}\rotatebox{53.64193519}{\makebox(0,0)[lb]{\smash{WALL}}}}%
                \put(0.40851564,0.35475138){\color[rgb]{1,1,1}\rotatebox{-53.87731588}{\makebox(0,0)[lb]{\smash{WALL}}}}%
                
                \put(0.07056856,0.33904472){\color[rgb]{0,0,0}\makebox(0,0)[lb]{\smash{$\text{r}$}}}%
                \put(0.17059229,0.13264657){\color[rgb]{0,0,0}\makebox(0,0)[lb]{\smash{$E_p$}}}%
                
                \put(0.67896659,0.6774319){\color[rgb]{0,0,0}\makebox(0,0)[lb]{\smash{$\text{r}$}}}%
                \put(0.6372626,0.64250297){\color[rgb]{0,0,0}\makebox(0,0)[lb]{\smash{$E_p$}}}%
            
    }
    \end{picture}%
\endgroup%
    \caption{Illustration of the different regions used along the paper.}
    \label{fig:regions}
\end{figure}

The following notation is introduced to define the different regions used along the paper. These regions are illustrated in Figure~\ref{fig:regions}.

Let $\Ball_\saferadious(\position)$ be a ball in $\mathbb{R}^3$ of radius $\saferadious=\saferadious_{\text{r}} + E_p
$ centered at a point $p$, where $\saferadious_{\text{r}}$ is the radius of the smallest ball containing the robot and $E_p$ is a safety distance. Continuous-time is represented by $\timme$ while $\disc$ represents discrete-time. The frame rate of the 3D sensor is denoted as $\framerate$, and the field of view at the $\disc$-th frame, which is assumed to be a frustum, is denoted by $\frustum[\disc]$. The configuration space is a bounded 3D cubic volume $\volume\subset\mathbb{R}^3$. The obstacle region is denoted by $\occupiedspace \subset \volume$ which characterizes the environment. Portions of $\occupiedspace$ that are sensed to be occupied are added to the volume $\occupiedspaceMem[\disc]$, when a portion of $\occupiedspaceMem[\disc]$ is sensed to be free it is removed from $\occupiedspaceMem[\disc]$. Thus, $\occupiedspaceMem$ represents the volume that is assumed to be occupied. Let $\castspace[\disc]$ be the volume swept by all the rays originating at the sensor and terminating either at a detected point or at the border of the $\frustum[\disc]$. $\freespace[\disc]=\{\position\in\castspace[\disc]:\Ball_\saferadious(\position)\cap\occupiedspaceMem[\disc] = \emptyset\}$, represents the volume inside the field of view that is free. The volume that is assumed to be traversable is denoted by $\navigationspace[\disc]=\{\position\in\volume:\Ball_\saferadious(\position)\cap
\occupiedspaceMem[\disc] = \emptyset\}$. Moreover, let $\mathcal{V}_{\text{safe}}[\disc]=\{\position\in\castspace[\disc]:\Ball_{R_\text{r}}(\position)\cap\occupiedspaceMem = \emptyset\}$ be a volume in $\castspace[\disc]$ containing all points with a distance greater than $R_\text{r}$ to any obstacle. Note that as long as the robot remains inside $\mathcal{V}_{\text{safe}}[\disc]$ it is collision free. Our path planning is constrained to sampling points inside $\navigationspace[\disc]$, since the trajectory generation considers a maximum deviation of $E_p$ as explained in Section \ref{Sec:Planning}.

A voxel $\voxel$ represents the space contained in a cubic volume, obtained by recursively dividing $\volume$ into eight equally sized volumes. The smallest voxel edge length is $\voxelrad$.

The position, linear velocity, linear acceleration, linear jerk and linear snap are denoted by $p,v,a,j,\snap\in\mathbb{R}^3$, respectively. The yaw orientation and its derivatives, angular velocity, angular acceleration, are denoted by $\yaw,\omega,\alpha\in \mathbb{R}$, respectively. 

%A path $\eta:[0,\dots,S]\to\navigationspace[\disc]$ is a sequence of position nodes. Given a path $\eta$, $\eta(s)$ is obtained as the concatenation of the line segments connecting consecutive nodes. 
A trajectory is defined as the vector $[p(t)^T,\ \yaw(t)]^T$ containing a time-dependant sequence of positions and yaw orientations, where $[\cdot]^T$ is the transpose operator.  The augmented state for the trajectory is defined as $z(t) = [p(t)^T,\yaw(t),v(t)^T,\omega(t),a(t)^T,j(t)^T]^T$. The set of feasible states is defined as 
\begin{equation}
\sallowed[\disc]=\left\{\state\ :
\begin{array}{l}
\position \in\mathcal{V}_{\text{safe}}[\disc],\\
\|v\|_\infty\leq\vmax,\  
\|a\|_\infty\leq\amax,\\ 
\|j\|_\infty\leq\jmax,\
|\omega|\leq\omegamax,\ 
\end{array}
\right\}.
\label{Eq:Zsafe}
\end{equation}

\subsection{Problem Definition}
%% agregar que tenemos un fov limitado porque no podemos cargar muchos sensores para ahorrar energia

\begin{problem}
Consider an holonomic robot modeled as a ball $\Ball_\saferadious(\position(\timme))$ with dynamics described by $\dot{\position}(\timme)=\velocity(\timme)$, $\dot{\velocity}(\timme)=\acceleration(\timme)$, $\dot{\acceleration}(\timme) = \jerk(\timme)$, $\dot{j}(\timme) = \snap(\timme)$, $\dot{\yaw}(\timme)=\angvelocity(\timme)$ and $\dot{\angvelocity}(\timme) = \angacceleration(\timme)$. 
Given an initial and final positions $\position_\start, \position_\final \notin \{\position\in\volume:\Ball_\saferadious(\position)\cap \occupiedspace = \emptyset \}$, and initial and final yaws $\yaw_\start, \yaw_\final$, respectively, navigate from $\state(\timme_\start)=[\position_\start^T, \yaw_\start, 0, \cdots, 0]^T$ to $\state(\timme_\final)=[\position_\final^T, \yaw_\final, 0, \cdots, 0]^T$ such that $\state(\timme)\in\sallowed[k]$,  $\forall t\in[t_i,t_f]$ and $k = \lfloor (t~-~t_i) \framerate \rfloor$, with some finite-time $t_f$.
\label{ProbPlanning}
\end{problem}

\begin{figure}
    \centering
    %\includegraphics[scale=0.7]{images/path_diagram.png}
    %\vspace{-0.7in}
    \def\svgwidth{10cm}
\begingroup%
  \makeatletter%
  \providecommand\color[2][]{%
    \errmessage{(Inkscape) Color is used for the text in Inkscape, but the package 'color.sty' is not loaded}%
    \renewcommand\color[2][]{}%
  }%
  \providecommand\transparent[1]{%
    \errmessage{(Inkscape) Transparency is used (non-zero) for the text in Inkscape, but the package 'transparent.sty' is not loaded}%
    \renewcommand\transparent[1]{}%
  }%
  \providecommand\rotatebox[2]{#2}%
  \ifx\svgwidth\undefined%
    \setlength{\unitlength}{441.87478638bp}%
    \ifx\svgscale\undefined%
      \relax%
    \else%
      \setlength{\unitlength}{\unitlength * \real{\svgscale}}%
    \fi%
  \else%
    \setlength{\unitlength}{\svgwidth}%
  \fi%
  \global\let\svgwidth\undefined%
  \global\let\svgscale\undefined%
  \makeatother%
  \begin{picture}(1,0.4280068)%  
      \put(0,0){\includegraphics[width=\unitlength,page=1]{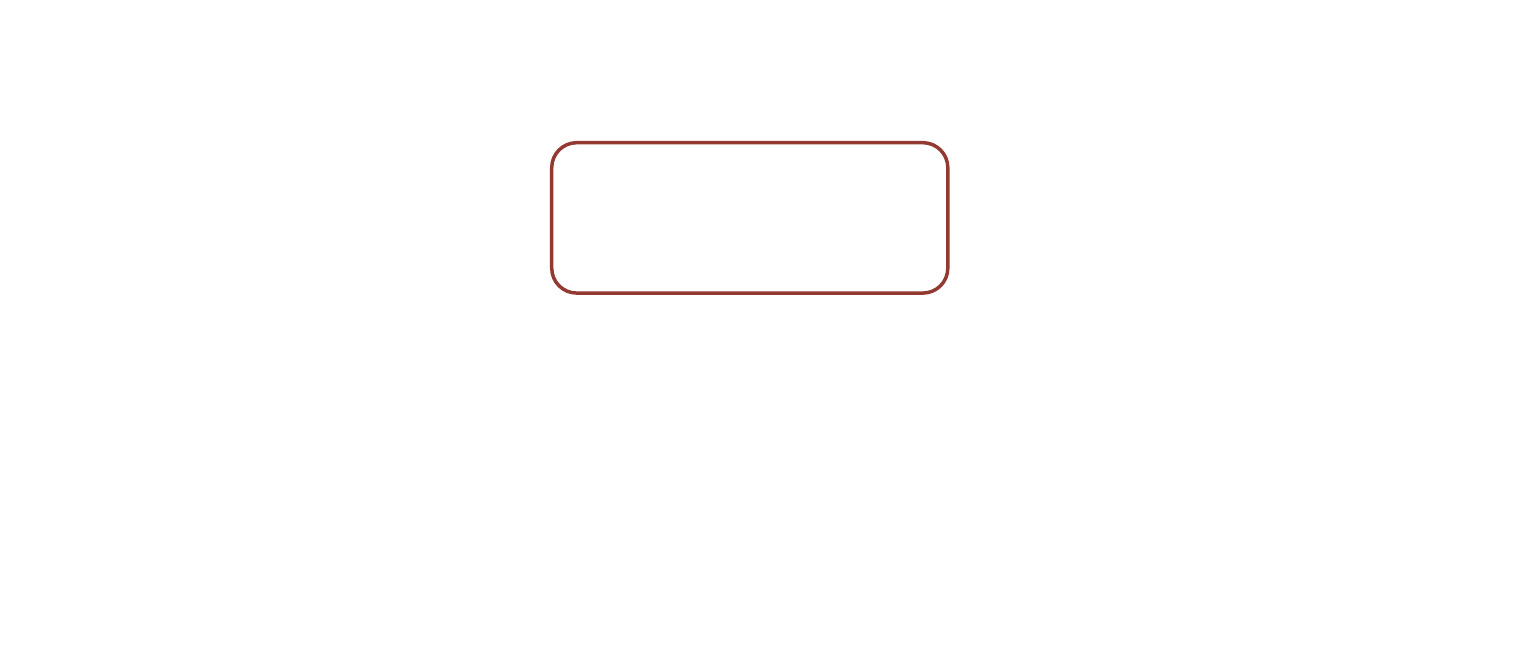}}%
      \put(0,0){\includegraphics[width=\unitlength,page=2]{images/main_diagram.pdf}}%
      \put(0,0){\includegraphics[width=\unitlength,page=3]{images/main_diagram.pdf}}%
      \put(0,0){\includegraphics[width=\unitlength,page=4]{images/main_diagram.pdf}}%
      \put(0,0){\includegraphics[width=\unitlength,page=5]{images/main_diagram.pdf}}%
      \put(0,0){\includegraphics[width=\unitlength,page=6]{images/main_diagram.pdf}}%
      \put(0,0){\includegraphics[width=\unitlength,page=7]{images/main_diagram.pdf}}%
      \put(0,0){\includegraphics[width=\unitlength,page=8]{images/main_diagram.pdf}}%
      \put(0,0){\includegraphics[width=\unitlength,page=9]{images/main_diagram.pdf}}%
      \put(0,0){\includegraphics[width=\unitlength,page=10]{images/main_diagram.pdf}}%
      \put(0,0){\includegraphics[width=\unitlength,page=11]{images/main_diagram.pdf}}%
      \put(0,0){\includegraphics[width=\unitlength,page=12]{images/main_diagram.pdf}}%
      \put(0,0){\includegraphics[width=\unitlength,page=13]{images/main_diagram.pdf}}%
      \put(0,0){\includegraphics[width=\unitlength,page=14]{images/main_diagram.pdf}}%
      \put(0,0){\includegraphics[width=\unitlength,page=15]{images/main_diagram.pdf}}%
      \scriptsize{
                  
                    \put(0.381,0.29210636){\color[rgb]{0,0,0}\makebox(0,0)[lb]{\smash{Map representation}}}%
                    \put(0.437,0.26210636){\color[rgb]{0,0,0}\makebox(0,0)[lb]{\smash{@33Hz}}}%
                    \put(0.408963110265,0.06806056){\color[rgb]{0,0,0}\makebox(0,0)[lb]{\smash{Path Planning}}}%
                    \put(0.435957110265,0.03806056){\color[rgb]{0,0,0}\makebox(0,0)[lb]{\smash{@33Hz}}}%
                    \put(0.726258033811,0.06806056){\color[rgb]{0,0,0}\makebox(0,0)[lb]{\smash{Trajectory Generation}}}%
                    \put(0.785033811,0.03806056){\color[rgb]{0,0,0}\makebox(0,0)[lb]{\smash{@240Hz}}}%
                    \put(0.7425210071,0.400497992){\color[rgb]{0,0,0}\makebox(0,0)[lb]{\smash{Control Algorithm}}}%
                    \put(0.78210071,0.37497992){\color[rgb]{0,0,0}\makebox(0,0)[lb]{\smash{@240Hz}}}%
                    \put(0.02,0.390563331){\color[rgb]{0,0,0}\makebox(0,0)[lb]{\smash{Odometry Sensor}}}%
                    \put(0.06598161,0.36363331){\color[rgb]{0,0,0}\makebox(0,0)[lb]{\smash{@240Hz}}}%
                    \put(0.0435,0.218149977){\color[rgb]{0,0,0}\makebox(0,0)[lb]{\smash{Depth Sensor}}}%
                    \put(0.075598161,0.1880149977){\color[rgb]{0,0,0}\makebox(0,0)[lb]{\smash{@33Hz}}}%
                    \put(0.0640512838,0.052515383){\color[rgb]{0,0,0}\makebox(0,0)[lb]{\smash{Global Goal}}}%
                    \put(0.26540877,0.199054077){\color[rgb]{0,0,0}\makebox(0,0)[lb]{\smash{disparity}}}%
                    \put(0.27940877,0.172054077){\color[rgb]{0,0,0}\makebox(0,0)[lb]{\smash{image}}}%
                    \put(0.2898469,0.07037314){\color[rgb]{0,0,0}\makebox(0,0)[lb]{\smash{$\position_\final$}}}%
                    \put(0.2798469,0.04037314){\color[rgb]{0,0,0}\makebox(0,0)[lb]{\smash{ $\yaw_\final$}}}%
                    \put(0.43936464,0.17178631){\color[rgb]{0,0,0}\makebox(0,0)[lb]{\smash{ $\occupiedspaceMem$}}}%
                    \put(0.61230269,0.0648314){\color[rgb]{0,0,0}\makebox(0,0)[lb]{\smash{ next }}}%
                    \put(0.60330269,0.0398314){\color[rgb]{0,0,0}\makebox(0,0)[lb]{\smash{ action }}}%
                    \put(0.81517822,0.19339566){\color[rgb]{0,0,0}\makebox(0,0)[lb]{\smash{ $\state(\timme)$ }}}%
                    \put(0.58183876,0.12272815){\color[rgb]{0,0,0}\makebox(0,0)[lb]{\smash{ motion planning}}}%
                    
                    \put(0.26782178,0.3554764){\color[rgb]{0,0,0}\makebox(0,0)[lb]{\smash{$\position$, $\yaw$}}}%

    }
    \tiny{
            %  \put(0.830821046,0.48633186){\color[rgb]{0,0,0}\makebox(0,0)[lb]{\smash{false}}}%
            %  \put(0.7770837547,0.18253539){\color[rgb]{0,0,0}\makebox(0,0)[lb]{\smash{true}}}%
            %  \put(0.65674906,0.34328218){\color[rgb]{0,0,0}\makebox(0,0)[lb]{\smash{$\waypoint_4$}}}%
            %  \put(0.327351458,0.18253539){\color[rgb]{0,0,0}\makebox(0,0)[lb]{\smash{$\waypoint_4$}}}%
            %  \put(0.1679468,0.06506758){\color[rgb]{0,0,0}\makebox(0,0)[lb]{\smash{true}}}%
            %  \put(0.01,0.26381483){\color[rgb]{0,0,0}\makebox(0,0)[lb]{\smash{false}}}%
             
            %  \put(0.90335477,0.520633186){\color[rgb]{0,0,0}\makebox(0,0)[lb]{\smash{$\position_\current = \waypoint_3$}}}%
            %  \put(0.9450946067,0.490633186){\color[rgb]{0,0,0}\makebox(0,0)[lb]{\smash{$\position_\final $}}}%
             
            %  \put(0.140948362,0.79715741){\color[rgb]{0,0,0}\makebox(0,0)[lb]{\smash{$\position_\current = \position_\start$}}}%
            %  \put(0.170788526,0.760693177){\color[rgb]{0,0,0}\makebox(0,0)[lb]{\smash{$\position_\final $}}}%
            %  \put(0.414885529,0.4130759){\color[rgb]{0,0,0}\makebox(0,0)[lb]{\smash{Analyze $\waypoint_4$}}}%
    }
    \end{picture}%
\endgroup%

    \caption{Illustration of the framework proposed to solve Problem~\ref{ProbPlanning}.}
    \label{fig:outline_solution}
\end{figure}
% \begin{figure}
%     \centering
%     \includegraphics[scale=0.7]{images/main_diagram.png}
%     \caption{Caption}
%     \label{fig:outline_solution}
% \end{figure}

\section{Reaching a goal in an unknown environment}

The outline of the solution proposed to solve Problem~\ref{ProbPlanning} is shown in Figure \ref{fig:outline_solution}. Since the robot is navigating in an unknown environment, we propose a framework divided in three main elements: Map representation, Path planning and Trajectory generation. 

First, at the map representation stage, a point cloud is computed taking as an input a disparity image and the odometry of the robot and added to a map representation of the occupied space $\occupiedspaceMem[\disc]$. Note that at any time $k$ the contents of the map inside $\frustum[\disc]$ may change. This may be due to measurement errors at the current frame or sensor drift errors accumulated over time. When a voxel previously marked as occupied is observed to be free, the $\occupiedspaceMem[\disc]$ is updated. Second, at the path planning stage, the next exploration action is generated by creating a path inside $\navigationspace[\disc]$, using a variation of the RRT connect, where one tree is expanded using the current position as root and sampling only inside  $\freespace[\disc]$, and it is connected with the other tree expanded from the goal used as root, which samples in $\navigationspace[\disc]$. In this way, the next position is always contained in $\freespace[\disc]$. Finally with the next planned action obtained, a trajectory is generated in such a way that drives the robot from its current state to the next planned action taking into account the robot's dynamical constraints. The yaw is controlled in such a way that the robot always moves inside $\frustum[\disc]$.

% In order to reject the occupied regions in the path planning algorithm, an octree voxel hashing map representation is used.

% With a given global goal a path planning formulation is posed and solved using a variation of the RRT-Connect algorithm which is modified to enforce the robot to always navigate inside the field of view. The path produced is used to generate the next exploration pose that allows the creation of a trajectory that drives the robot from the current pose to the next exploration pose while fulfilling the dynamical constraints of the robot. This process constitute the motion planning stage and it is repeated until the global goal is reached. 

\subsection{Map Representation} \label{SecMapping}

% \begin{figure}
%     \centering
%     \includegraphics[scale=0.6]{images/maqpping_diagram.png}
%     \caption{Caption}
%     \label{fig:diag_map}
% \end{figure}
\begin{figure}
    \centering
    \vspace{0.5in}
    \def\svgwidth{10cm}
\begingroup%
  \makeatletter%
  \providecommand\color[2][]{%
    \errmessage{(Inkscape) Color is used for the text in Inkscape, but the package 'color.sty' is not loaded}%
    \renewcommand\color[2][]{}%
  }%
  \providecommand\transparent[1]{%
    \errmessage{(Inkscape) Transparency is used (non-zero) for the text in Inkscape, but the package 'transparent.sty' is not loaded}%
    \renewcommand\transparent[1]{}%
  }%
  \providecommand\rotatebox[2]{#2}%
  \ifx\svgwidth\undefined%
    \setlength{\unitlength}{976.65446604382.37031335bp}%
    \ifx\svgscale\undefined%
      \relax%
    \else%
      \setlength{\unitlength}{\unitlength * \real{\svgscale}}%
    \fi%
  \else%
    \setlength{\unitlength}{\svgwidth}%
  \fi%
  \global\let\svgwidth\undefined%
  \global\let\svgscale\undefined%
  \makeatother%
  \begin{picture}(1,0.7425855389698368)%  
      \put(0,0){\includegraphics[width=\unitlength,page=1]{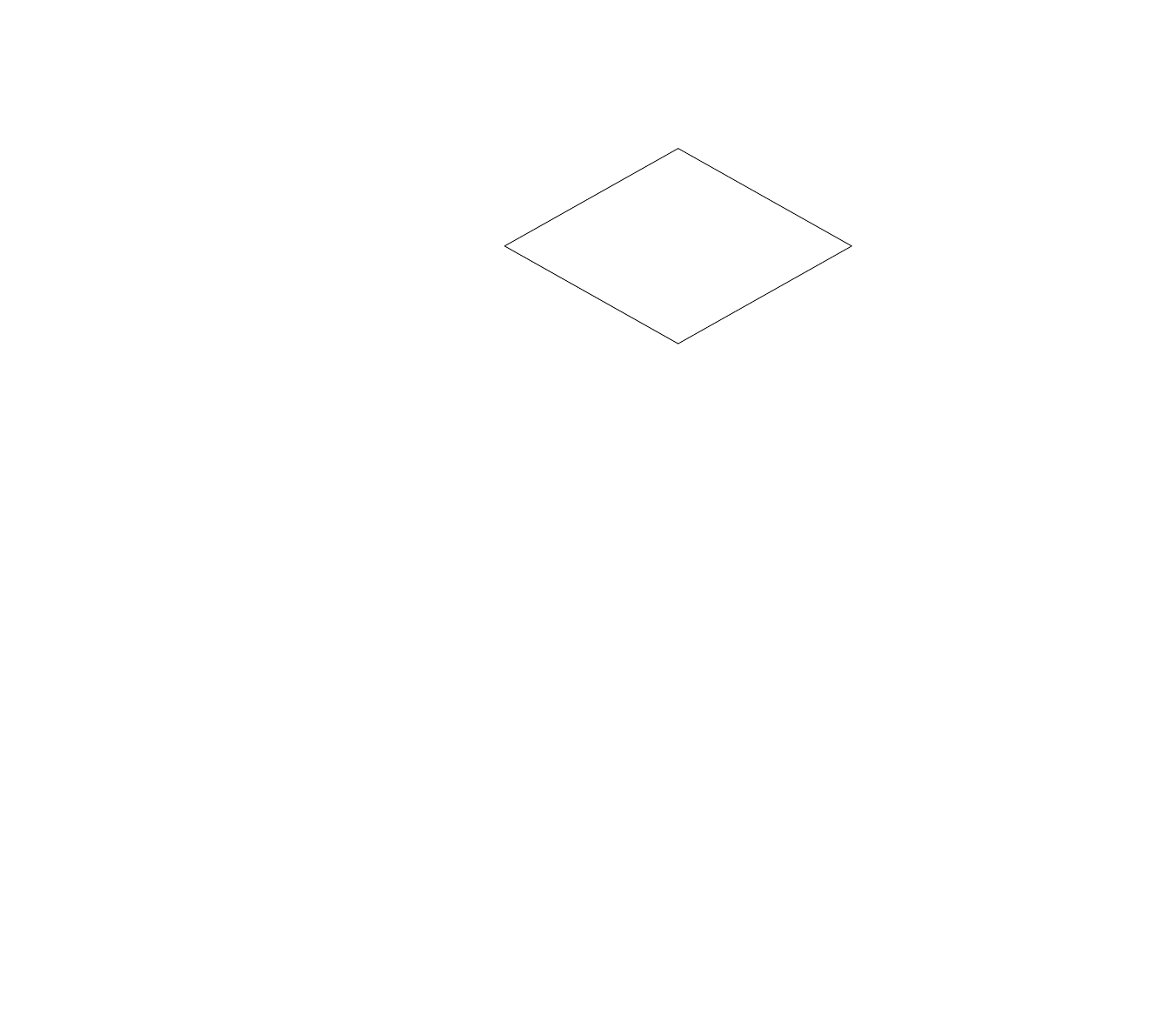}}%
      \put(0,0){\includegraphics[width=\unitlength,page=2]{images/maqpping_diagram.pdf}}%
      \put(0,0){\includegraphics[width=\unitlength,page=3]{images/maqpping_diagram.pdf}}%
      \put(0,0){\includegraphics[width=\unitlength,page=4]{images/maqpping_diagram.pdf}}%
      \put(0,0){\includegraphics[width=\unitlength,page=5]{images/maqpping_diagram.pdf}}%
      \put(0,0){\includegraphics[width=\unitlength,page=6]{images/maqpping_diagram.pdf}}%
      \put(0,0){\includegraphics[width=\unitlength,page=7]{images/maqpping_diagram.pdf}}%
      \put(0,0){\includegraphics[width=\unitlength,page=8]{images/maqpping_diagram.pdf}}%
      \put(0,0){\includegraphics[width=\unitlength,page=9]{images/maqpping_diagram.pdf}}%
      \put(0,0){\includegraphics[width=\unitlength,page=10]{images/maqpping_diagram.pdf}}%
      \put(0,0){\includegraphics[width=\unitlength,page=11]{images/maqpping_diagram.pdf}}%
      \put(0,0){\includegraphics[width=\unitlength,page=12]{images/maqpping_diagram.pdf}}%
      \put(0,0){\includegraphics[width=\unitlength,page=13]{images/maqpping_diagram.pdf}}%
      \put(0,0){\includegraphics[width=\unitlength,page=14]{images/maqpping_diagram.pdf}}%
      \put(0,0){\includegraphics[width=\unitlength,page=15]{images/maqpping_diagram.pdf}}%
      \put(0,0){\includegraphics[width=\unitlength,page=16]{images/maqpping_diagram.pdf}}%
      \put(0,0){\includegraphics[width=\unitlength,page=17]{images/maqpping_diagram.pdf}}%
      
      \footnotesize{
                \put(0.48998456,0.69178211){\color[rgb]{0,0,0}\makebox(0,0)[lb]{\smash{Points in \textbf{\textit{Point}}}}}%  
                \put(0.5008456,0.66178211){\color[rgb]{0,0,0}\makebox(0,0)[lb]{\smash{\textbf{\textit{Cloud List}} ?}}}%  
                \put(0.48865784,0.2206831){\color[rgb]{0,0,0}\makebox(0,0)[lb]{\smash{\textit{key.level} == $\maxlevel$ ?}}}% 
                \put(0.86344287,0.51278553){\color[rgb]{0,0,0}\makebox(0,0)[lb]{\smash{Update}}}% 
                \put(0.835344287,0.48278553){\color[rgb]{0,0,0}\makebox(0,0)[lb]{\smash{probability}}}%
                \put(0.06267507,0.2370701){\color[rgb]{0,0,0}\makebox(0,0)[lb]{\smash{Remove \textit{key} from}}}%
                \put(0.14267507,0.2070701){\color[rgb]{0,0,0}\makebox(0,0)[lb]{\smash{\textbf{\textit{Map}}}}}%
                \put(0.43550374,0.04020973){\color[rgb]{0,0,0}\makebox(0,0)[lb]{\smash{\textit{key.level} $\leftarrow$ \textit{key.level}$+1$}}}% 
                \put(0.44928134,0.49708604){\color[rgb]{0,0,0}\makebox(0,0)[lb]{\smash{Pop a point from \textbf{\textit{list}}}}}%
                \put(0.42242528,0.39327047){\color[rgb]{0,0,0}\makebox(0,0)[lb]{\smash{\textit{key} $\leftarrow$ find point in \textbf{\textit{Map}}}}}%
                \put(0.03613386,0.0402097){\color[rgb]{0,0,0}\makebox(0,0)[lb]{\smash{Add \textit{children} to \textbf{\textit{Map}}}}}%
                \put(0.05597345,0.65484935){\color[rgb]{0,0,0}\makebox(0,0)[lb]{\smash{Wait new disparity}}}%
                \put(0.13597345,0.62484935){\color[rgb]{0,0,0}\makebox(0,0)[lb]{\smash{image}}}%
                \put(0.07365264,0.49708604){\color[rgb]{0,0,0}\makebox(0,0)[lb]{\smash{Get point cloud}}}%
                 \put(0.46046577,0.86527876){\color[rgb]{0,0,0}\makebox(0,0)[lb]{\smash{Grouped raycasting}}}%
                 \put(0.02047617,0.31232052){\color[rgb]{0,0,0}\makebox(0,0)[lb]{\smash{Update tree}}}%
                 \put(0.14818194,0.73911319){\color[rgb]{0,0,0}\makebox(0,0)[lb]{\smash{start}}}%

    }
    \scriptsize{
            \put(0.33293488,0.22115441){\color[rgb]{0,0,0}\makebox(0,0)[lb]{\smash{false}}}%
            \put(0.56954383,0.5565206){\color[rgb]{0,0,0}\makebox(0,0)[lb]{\smash{true}}}%
            \put(0.88433807,0.33418618){\color[rgb]{0,0,0}\makebox(0,0)[lb]{\smash{true}}}%
            \put(0.56032705,0.80168078){\color[rgb]{0,0,0}\makebox(0,0)[lb]{\smash{false}}}%
            
    }
    \end{picture}%
\endgroup%

    \caption{Flow diagram of the mapping algorithm.}
    \label{fig:diag_map}
\end{figure}

Recent work has advocated the use of spatial hash tables to represent the environment. Hash tables provide a sparse representation that can be dynamically expanded with constant time insertion and look up. Following a spatial hashing approach similar to \cite{Niener2013Real-timeHashing}, \cite{Oleynikova2018SafeVehicles} showed that voxel hashing can be used advantageously for robotic exploration and mapping. As an alternative to voxel hashing, hierarchical data structures, such as octrees, sacrifice raw speed for single queries but their additional structure can be advantageous for planning, as shown for example in~\cite{Vespa2018EfficientMapping}.

In this work, a linear octree is used to represent the map of the environment. Linear octrees trade-off reduced memory usage for speed by storing only leaf nodes instead of all tree nodes as is the case for regular octrees, while preserving the hierarchical structure~\cite{Gargantini1982LinearObjects}. 

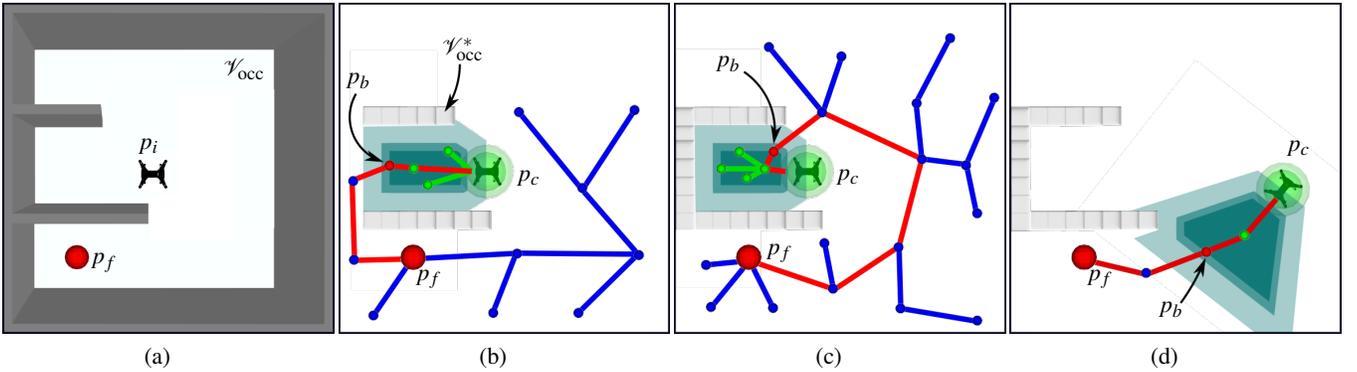
\begin{figure*}
    \centering
    \subfloat[]{\label{subfig:scene}

\def\svgwidth{4.4cm}
\begingroup%
  \makeatletter%
  \providecommand\color[2][]{%
    \errmessage{(Inkscape) Color is used for the text in Inkscape, but the package 'color.sty' is not loaded}%
    \renewcommand\color[2][]{}%
  }%
  \providecommand\transparent[1]{%
    \errmessage{(Inkscape) Transparency is used (non-zero) for the text in Inkscape, but the package 'transparent.sty' is not loaded}%
    \renewcommand\transparent[1]{}%
  }%
  \providecommand\rotatebox[2]{#2}%
  \ifx\svgwidth\undefined%
    \setlength{\unitlength}{350.64566926bp}%
    \ifx\svgscale\undefined%
      \relax%
    \else%
      \setlength{\unitlength}{\unitlength * \real{\svgscale}}%
    \fi%
  \else%
    \setlength{\unitlength}{\svgwidth}%
  \fi%
  \global\let\svgwidth\undefined%
  \global\let\svgscale\undefined%
  \makeatother%
  \begin{picture}(1,1)%  
      \put(0,0){\includegraphics[width=\unitlength]{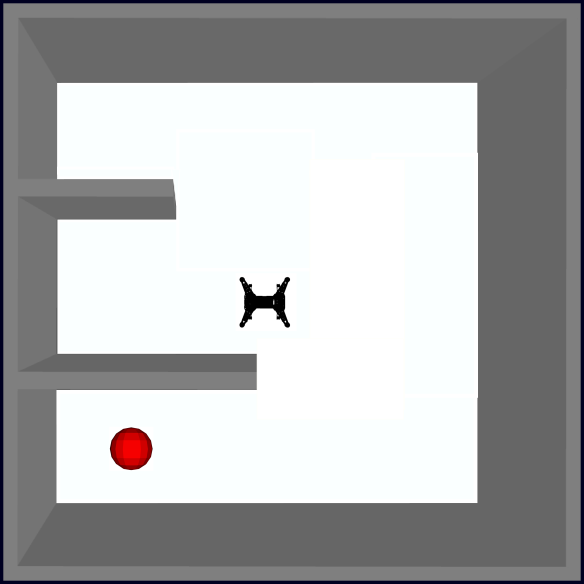}}%
      %\put(0,0){\includegraphics[width=\unitlength,page=2]{images/step0.pdf}}%
  \footnotesize{
    \put(0.40949223,0.55843928){\color[rgb]{0,0,0}\makebox(0,0)[lb]{\smash{$\position_\start$}}}%
    \put(0.26645945,0.21199113){\color[rgb]{0,0,0}\makebox(0,0)[lb]{\smash{$\position_\final$}}}%
    \put(0.665,0.78){\color[rgb]{0,0,0}\makebox(0,0)[lb]{\smash{$\occupiedspace$}}}%
    }  \end{picture}%
\endgroup%    
    }
    \hspace{-0.38cm}
    \subfloat[]{\label{subfig:first_step}
    \def\svgwidth{4.4cm}
\begingroup%
  \makeatletter%
  \providecommand\color[2][]{%
    \errmessage{(Inkscape) Color is used for the text in Inkscape, but the package 'color.sty' is not loaded}%
    \renewcommand\color[2][]{}%
  }%
  \providecommand\transparent[1]{%
    \errmessage{(Inkscape) Transparency is used (non-zero) for the text in Inkscape, but the package 'transparent.sty' is not loaded}%
    \renewcommand\transparent[1]{}%
  }%
  \providecommand\rotatebox[2]{#2}%
  \ifx\svgwidth\undefined%
    \setlength{\unitlength}{350.64566926bp}%
    \ifx\svgscale\undefined%
      \relax%
    \else%
      \setlength{\unitlength}{\unitlength * \real{\svgscale}}%
    \fi%
  \else%
    \setlength{\unitlength}{\svgwidth}%
  \fi%
  \global\let\svgwidth\undefined%
  \global\let\svgscale\undefined%
  \makeatother%
  \begin{picture}(1,1)%
      \put(0,0){\includegraphics[width=\unitlength,page=1]{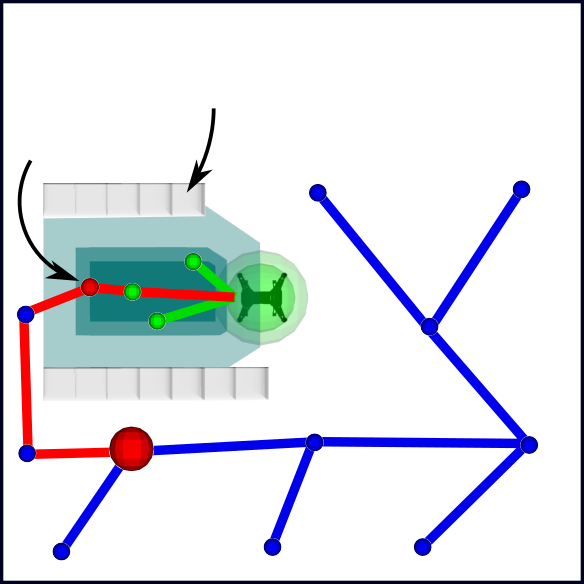}}%    
      %\put(0,0){\includegraphics[width=\unitlength,page=2]{images/step1.pdf}}%
      %\put(0,0){\includegraphics[width=\unitlength,page=3]{images/step1.pdf}}%
      %\put(0,0){\includegraphics[width=\unitlength,page=4]{images/step1.pdf}}%
    \footnotesize{
    \put(0.53861084,0.46227653){\color[rgb]{0,0,0}\makebox(0,0)[lb]{\smash{$\position_\current$}}}%
    \put(0.02278241,0.75620389){\color[rgb]{0,0,0}\makebox(0,0)[lb]{\smash{$\position_b$}}}%
    \put(0.23319483,0.16237841){\color[rgb]{0,0,0}\makebox(0,0)[lb]{\smash{$\position_\final$}}}%
    \put(0.31154674,0.84354072){\color[rgb]{0,0,0}\makebox(0,0)[lb]{\smash{$\occupiedspaceMem$}}}%
    
    }
  \end{picture}%
\endgroup%    
    }\hspace{-0.38cm}
    %\vspace{-0.35cm}
    \subfloat[]{\label{subfig:second_step}
    \def\svgwidth{4.4cm}
\begingroup%
  \makeatletter%
  \providecommand\color[2][]{%
    \errmessage{(Inkscape) Color is used for the text in Inkscape, but the package 'color.sty' is not loaded}%
    \renewcommand\color[2][]{}%
  }%
  \providecommand\transparent[1]{%
    \errmessage{(Inkscape) Transparency is used (non-zero) for the text in Inkscape, but the package 'transparent.sty' is not loaded}%
    \renewcommand\transparent[1]{}%
  }%
  \providecommand\rotatebox[2]{#2}%
  \ifx\svgwidth\undefined%
    \setlength{\unitlength}{350.64566926bp}%
    \ifx\svgscale\undefined%
      \relax%
    \else%
      \setlength{\unitlength}{\unitlength * \real{\svgscale}}%
    \fi%
  \else%
    \setlength{\unitlength}{\svgwidth}%
  \fi%
  \global\let\svgwidth\undefined%
  \global\let\svgscale\undefined%
  \makeatother%
  \begin{picture}(1,1)%
      \put(0,0){\includegraphics[width=\unitlength,page=1]{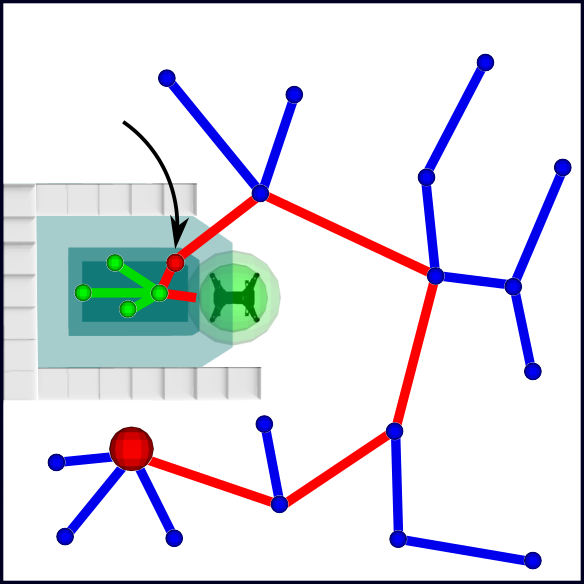}}
      %\put(0,0){\includegraphics[width=\unitlength,page=2]{images/step2.pdf}}
      %\put(0,0){\includegraphics[width=\unitlength,page=3]{images/step2.pdf}}
    \footnotesize{
    \put(0.48999677,0.46551742){\color[rgb]{0,0,0}\makebox(0,0)[lb]{\smash{$\position_c$}}}%
    \put(0.12620498,0.80826862){\color[rgb]{0,0,0}\makebox(0,0)[lb]{\smash{$\position_b$}}}%
     \put(0.27848115,0.2402032){\color[rgb]{0,0,0}\makebox(0,0)[lb]{\smash{$\position_\final$}}}%
    }
  \end{picture}%
\endgroup%    
    }\hspace{-0.38cm}
    %\vspace{-0.35cm}
    \subfloat[]{\label{subfig:third_step}
    \def\svgwidth{4.4cm}
\begingroup%
  \makeatletter%
  \providecommand\color[2][]{%
    \errmessage{(Inkscape) Color is used for the text in Inkscape, but the package 'color.sty' is not loaded}%
    \renewcommand\color[2][]{}%
  }%
  \providecommand\transparent[1]{%
    \errmessage{(Inkscape) Transparency is used (non-zero) for the text in Inkscape, but the package 'transparent.sty' is not loaded}%
    \renewcommand\transparent[1]{}%
  }%
  \providecommand\rotatebox[2]{#2}%
  \ifx\svgwidth\undefined%
    \setlength{\unitlength}{350.64566926bp}%
    \ifx\svgscale\undefined%
      \relax%
    \else%
      \setlength{\unitlength}{\unitlength * \real{\svgscale}}%
    \fi%
  \else%
    \setlength{\unitlength}{\svgwidth}%
  \fi%
  \global\let\svgwidth\undefined%
  \global\let\svgscale\undefined%
  \makeatother%
  \begin{picture}(1,1)%
      \put(0,0){\includegraphics[width=\unitlength,page=1]{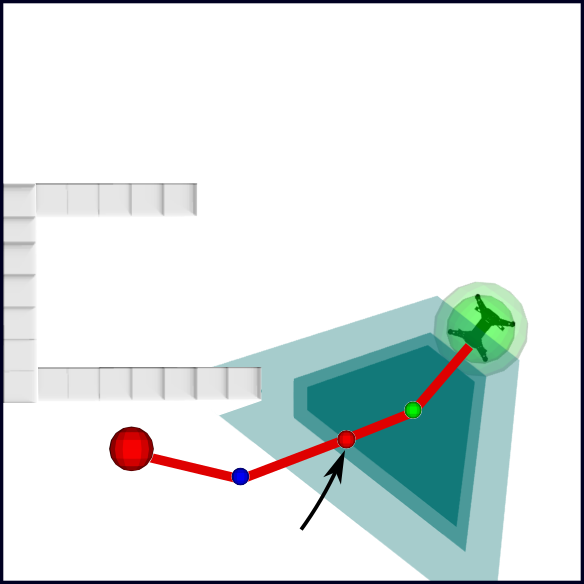}}%  
      %\put(0,0){\includegraphics[width=\unitlength,page=2]{images/step3.pdf}}%  
      %\put(0,0){\includegraphics[width=\unitlength,page=3]{images/step3.pdf}}%  
    \footnotesize{
        \put(0.83637728,0.55221232){\color[rgb]{0,0,0}\makebox(0,0)[lb]{\smash{$\position_\current$}}}%
        \put(0.44943857,0.06440364){\color[rgb]{0,0,0}\makebox(0,0)[lb]{\smash{$\position_b$}}}%
        \put(0.23319483,0.16237841){\color[rgb]{0,0,0}\makebox(0,0)[lb]{\smash{$\position_\final$}}}%
    }
  \end{picture}%
\endgroup%    
    
    }
    \caption{Illustration of the motion planning method for reaching a goal in an unknown environment: (a) Unknown environment and goal. (b) Initial proposed path using the current map. (c) Replanned path after updating the map. (d) Final path towards the goal.}
    \label{fig:method_example}
\end{figure*}

The map creation is illustrated in Fig.~\ref{fig:diag_map}. The \console{Map} is implemented as a linear octree with $\maxlevel$ levels and with axes-aligned bounding boxes. When a new disparity image is received, a point cloud is created by reprojecting occupied pixels to 3D-space in the world coordinate system (block \console{Get point cloud} in Figure~\ref{fig:diag_map}). For each point, a key is generated formed by a level code indicating the depth $\maplevel$ in the octree (i.e. $\console{key.level}=\maplevel$) and a spatial code computed by
\begin{equation}\label{eqn:key}
\console{key.spatial}(x,y,z,l) = \console{cat}\left\{\bar{x},\bar{y},\bar{z}\right\}
\end{equation}
where
\begin{equation}
\begin{array}{lll}
\bar{x} = \left\lfloor\frac{x}{2^{\maxlevel - \maplevel} \voxelrad}\right\rfloor, & \bar{y} = \left\lfloor\frac{y}{2^{\maxlevel - \maplevel} \voxelrad}\right\rfloor, & \bar{z} = \left\lfloor\frac{z}{2^{\maxlevel - \maplevel} \voxelrad}\right\rfloor,
\end{array}
\end{equation}
$\console{cat}\{\cdot,\cdot,\cdot\}$ concatenates its arguments represented in $m$ digits and $\lfloor\cdot\rfloor$ represents the floor operation. To find the key corresponding to the current location of the point in the map, equation~\eqref{eqn:key} is used to generate different keys from the root level $\maplevel=0$ up to $\maplevel=\maxlevel$ until the key is found in \console{Map} (block \console{find point in Map} in Figure~\ref{fig:diag_map}). As is common in octrees, \console{Map} insertion is logarithmic in the size of the map, which can be dynamically expanded as needed at practically no cost. For each level $\maplevel$, if the obtained key does not correspond to one in $\maplevel$, the key is removed from \console{Map} and the space represented by this voxel is subdivided into eight equal volume subspaces and their centers are used to generate new keys that are added to the \console{Map} (block \console{Add children in Map} in Figure~\ref{fig:diag_map}). Then, \console{key.level} is incremented. This process (block \console{Update tree} in Figure~\ref{fig:diag_map}) is repeated until \console{key.level} is equal to $\maxlevel-1$. Finally, when \console{key.level} is equal to $\maxlevel$ the occupancy probability of the corresponding voxel is increased using the stereo error model  from~\cite{Schauwecker2014RobustVision} (block \console{Update probability} in Figure~\ref{fig:diag_map}). This model has been found to provide an effective method to account for measurement errors, in particular temporally and spatially correlated errors which occur in stereo vision systems. When the voxel probability is larger than a certain threshold $\pocc$, it is considered as part of $\occupiedspaceMem[\disc]$. 

After processing all points in the point cloud, the grouped raycasting method, suggested by \cite{Oleynikova2017Voxblox:Planning}, is applied to all voxels in  $\occupiedspaceMem[\disc]\cap\frustum[\disc]$ with occupancy probability larger than a threshold $\pray$, with $\pray>\pocc$ (block \console{Grouped raycasting} in Figure~\ref{fig:diag_map}). This step consists on updating the occupancy probability of the voxels in $\castspace[k]$ %not detected by the current disparity image
by decreasing their occupancy probability according to the stereo error model~\cite{Schauwecker2014RobustVision}. 

Even though a flat grid representation of the map, i.e. only saving observed voxels at maximum discretization, would be faster for map creation, our linear octree implementation presents a significant advantage for collision checking between points during path planning. Notice that the linear octree is similar to a sparse grid representation with an additional block \console{Update tree} (see Figure~\ref{fig:diag_map}).
Also notice that all points inserted into the map must be inside $\frustum[\disc]$ and that typically there will be some overlap between $\frustum[\disc]$ and $\frustum[\disc+1]$ as the robot needs to update it's map frequently during navigation. Therefore, since the method \console{Update tree}    is only applied to points in $\frustum[\disc+1]\setminus\frustum[\disc]$ the worst case in terms of computation cost would be expected to happen only at the start when the first frame is processed. Nonetheless, in the worst case, that is, when $100\%$ of the points in $\frustum[\disc]$ are added to the map, the penalty in computation time to implement our linear octree has been found experimentally to be only a small fraction of complete perception pipeline. %  only~$\approx 0.3ms$. This difference was obtained by averaging the time required for adding points from 1000 consecutive $\frustum[\disc]$ to the two map representations.

\subsection{Motion Planning}
\label{Sec:Planning}

In our approach, the robot always has to move in its field of view, i.e., inside $\freespace[\disc]$, to ensure collision-free navigation. 
%This is required because, even if a region has been previously explored, the $\occupiedspaceMem$ may have changed due to measurement errors at the current frame or sensor drift errors accumulated over time. 
Te proposed motion planning algorithm is described in the flow diagram of Figure~\ref{fig:path_planning_diagram} and illustrated in Figure~\ref{fig:method_example}.

At the beginning, the robot has a starting position $\position_\start$, a goal position $\position_\final$, a current position $\position_\current = \position_\start$  and has no information about the map as illustrated in Figure~\ref{subfig:scene}. When a new disparity image from the depth sensor arrives, it is converted to a point cloud and serves as an input to the map representation described in Section~\ref{SecMapping}. The map representation, the current position $\position_\current$ and the goal position $\position_\final$ serve as the input to the path planning algorithm which is solved using a variation of the RRT-Connect \cite{Kuffner2000RRT-connect:Planning} algorithm (block \console{Solve~RRT-Connect} in Figure~\ref{fig:path_planning_diagram}). The solution of this problem generates a sequence of positions that connects through $\navigationspace[\disc]$ the current position $\position_\current$ and the goal $\position_\final$, as illustrated in Figure~\ref{subfig:first_step}.

The output of the RRT-Connect consists in a sequence of waypoints $\{p_c,p_2,p_3,\dots,p_f\}$ in $\navigationspace[\disc]$. 
Let 
$$
\waypoints[\disc] = \{[\position_\current^T,\psi_c], [\position_2^T,\psi_2], [\position_3^T,\psi_3], \dots, [\position_\final^T,\psi_f]\}.
$$ 
be a path in $\mathbb{R}^4$ from the robot's current position and orientation $[\position_\current^T,\psi_c]$ to the given goal configuration $[\position_\final^T,\psi_f]$, in such a way that all positions are elements of the RRT-Connect solution and each $\psi_i$ corresponds to the angle from the $x$-axis to the vector $p_{i+1}-p_i$. Choosing $\psi_i$ in this way ensures that at each waypoint the field of view of the robot is aligned with the direction of the next waypoint. % adding $2\pi n$ with $n\in\mathbb{Z}$ to minimze $|\psi_i-\psi_{i-1}|$. 
Hence, the robot always navigates in $\freespace[\disc]$. %In this way, a deviation $E_\yaw$ of less than half of the field of view, between the real orientation and the desired orientation $\yaw_i$ at a given waypoint is allowed, since still ensures that the next waypoint is in the field of view.

To execute the path, the subset $\tilde{\waypoints}[\disc]$ of $m$ waypoints in $\waypoints[\disc]$ together with the point $[p_b,\psi_m]$ where $p_b$ is the intersection between the boundary of $\freespace[\disc]$ and line segment connecting $p_m$ and $p_{m+1}$ (block \console{Get sub-path in $\freespace[\disc]$} in Figure~\ref{fig:path_planning_diagram}) is pushed to the trajectory planning algorithm which is explained in detail in Section~\ref{SecGen} (block \console{Send $\tilde{\waypoints}[\disc]$ to trajectory planning} in Figure~\ref{fig:path_planning_diagram}). Since $p_{m+1}$ is not in $\freespace[\disc]$, the point $p_b$ is added at the boundary of $\freespace[\disc]$, in this way $\tilde{\waypoints}[\disc]$ contains only points in $\freespace[\disc]$. Afterwards, $\position_c$ is updated as $\position_b$.
%, ensures that the trajectory of the robot will eventually reach a box $\wayregion$ of lenght $2\emax$  centered at each waypoint in the right order. 

%%in such a way that, the trajectory planning algorithm, explained in detail in Section~\ref{SecGen}, is always aware of the first three waypoints in $\waypoints$ (block \textit{``Send~to~trajectory~planning"} in Figure~\ref{fig:path_planning_diagram}). Once a waypoint is pushed to the trajectory planning algorithm, the trajectory of the robot will eventually reach a region $\wayregion$ around such waypoint. 

Whenever the map is updated, if the line segments connecting consecutive points in $\waypoints[\disc]$ are not $\occupiedspaceMem[\disc]$ (decision \console{is path in $\occupiedspaceMem[\disc]$ ?} in Figure~\ref{fig:path_planning_diagram}), then $\tilde{\waypoints}[\disc]$ is pushed to the trajectory planning and $\position_c$ is updated as before. Otherwise, a replanning is needed. Hence, a new RRT-Connect solution is obtained from $\position_c$ to $\position_\final$ as shown in Figure~\ref{subfig:second_step} (block \console{Solve~RRT-Connect} in Figure~\ref{fig:path_planning_diagram}). This process continues, exploring new regions as illustrated in Figure~\ref{subfig:third_step}, until the goal is reached. 

\begin{figure}
    \centering
    %\includegraphics[scale=0.7]{images/path_diagram.png}
    %\vspace{-0.7in}
    \def\svgwidth{8.2cm}
\begingroup%
  \makeatletter%
  \providecommand\color[2][]{%
    \errmessage{(Inkscape) Color is used for the text in Inkscape, but the package 'color.sty' is not loaded}%
    \renewcommand\color[2][]{}%
  }%
  \providecommand\transparent[1]{%
    \errmessage{(Inkscape) Transparency is used (non-zero) for the text in Inkscape, but the package 'transparent.sty' is not loaded}%
    \renewcommand\transparent[1]{}%
  }%
  \providecommand\rotatebox[2]{#2}%
  \ifx\svgwidth\undefined%
    \setlength{\unitlength}{215.8008116bp}%
    \ifx\svgscale\undefined%
      \relax%
    \else%
      \setlength{\unitlength}{\unitlength * \real{\svgscale}}%
    \fi%
  \else%
    \setlength{\unitlength}{\svgwidth}%
  \fi%
  \global\let\svgwidth\undefined%
  \global\let\svgscale\undefined%
  \makeatother%
  \begin{picture}(1,1.04021124)%  
      \put(0,0){\includegraphics[width=\unitlength,page=1]{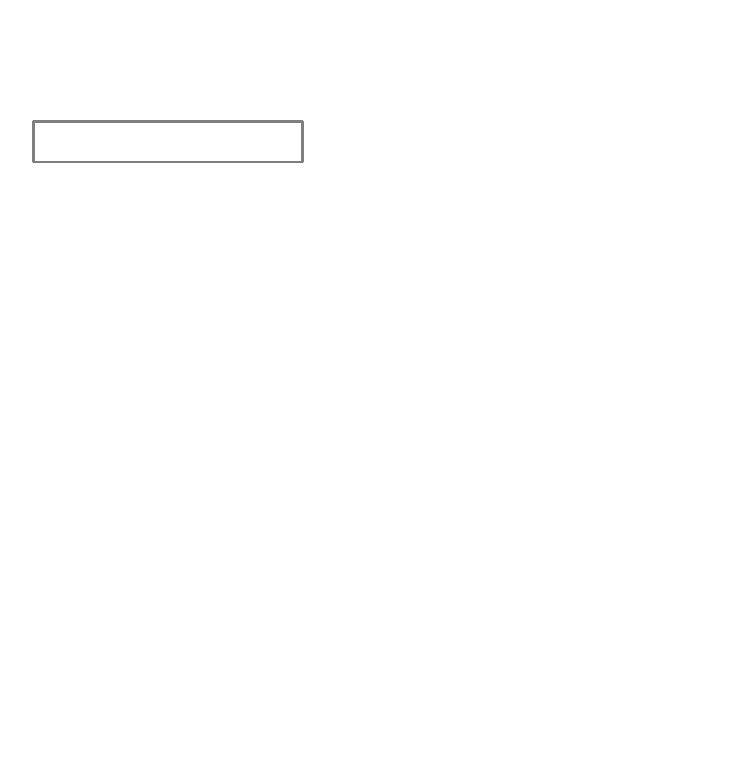}}%
      \put(0,0){\includegraphics[width=\unitlength,page=2]{images/path_diagram.pdf}}%
      \put(0,0){\includegraphics[width=\unitlength,page=3]{images/path_diagram.pdf}}%
      \put(0,0){\includegraphics[width=\unitlength,page=4]{images/path_diagram.pdf}}%
      \put(0,0){\includegraphics[width=\unitlength,page=5]{images/path_diagram.pdf}}%
      \put(0,0){\includegraphics[width=\unitlength,page=6]{images/path_diagram.pdf}}%
      \put(0,0){\includegraphics[width=\unitlength,page=7]{images/path_diagram.pdf}}%
      \put(0,0){\includegraphics[width=\unitlength,page=8]{images/path_diagram.pdf}}%
      \put(0,0){\includegraphics[width=\unitlength,page=9]{images/path_diagram.pdf}}%
      \put(0,0){\includegraphics[width=\unitlength,page=10]{images/path_diagram.pdf}}%
      \put(0,0){\includegraphics[width=\unitlength,page=11]{images/path_diagram.pdf}}%
      \put(0,0){\includegraphics[width=\unitlength,page=12]{images/path_diagram.pdf}}%
      \put(0,0){\includegraphics[width=\unitlength,page=13]{images/path_diagram.pdf}}%
      \put(0,0){\includegraphics[width=\unitlength,page=14]{images/path_diagram.pdf}}%
      \put(0,0){\includegraphics[width=\unitlength,page=15]{images/path_diagram.pdf}}%
      \put(0,0){\includegraphics[width=\unitlength,page=16]{images/path_diagram.pdf}}%
      \footnotesize{
                    \put(0.06473312,0.83823349){\color[rgb]{0,0,0}\makebox(0,0)[lb]{\smash{Solve RRT-Connect}}}%
                    \put(0.03205264,0.70090511){\color[rgb]{0,0,0}\makebox(0,0)[lb]{\smash{Get sub-path in $\freespace[\disc]$}}}%
                    \put(0.16268022,0.26541042){\color[rgb]{0,0,0}\makebox(0,0)[lb]{\smash{is $\position_\final$ in }}}%
                    \put(0.16268022,0.21541042){\color[rgb]{0,0,0}\makebox(0,0)[lb]{\smash{$\freespace[\disc]$ ?}}}%
                    \put(0.1936794,0.01611948){\color[rgb]{0,0,0}\makebox(0,0)[lb]{\smash{end}}}%
                    \put(0.13615787,0.5855){\color[rgb]{0,0,0}\makebox(0,0)[lb]{\smash{Send $\tilde{\waypoints}[\disc]$ to}}}%
                    \put(0.07015787,0.55){\color[rgb]{0,0,0}\makebox(0,0)[lb]{\smash{trajectory planning}}}%
                    \put(0.74550852,0.72094679){\color[rgb]{0,0,0}\makebox(0,0)[lb]{\smash{is path in }}}%
                    \put(0.75050852,0.67594679){\color[rgb]{0,0,0}\makebox(0,0)[lb]{\smash{$\occupiedspaceMem[\disc]$ ?}}}%
                    \put(0.74257776,0.25042703){\color[rgb]{0,0,0}\makebox(0,0)[lb]{\smash{Wait \textbf{\textit{Map}}}}}%
                    \put(0.77257776,0.2152703){\color[rgb]{0,0,0}\makebox(0,0)[lb]{\smash{update}}}%
                    \put(0.1651021,0.44091741){\color[rgb]{0,0,0}\makebox(0,0)[lb]{\smash{$\position_\current \leftarrow \position_b$}}}%
    }
    \scriptsize{
            \put(0.52010617,0.2364095){\color[rgb]{0,0,0}\makebox(0,0)[lb]{\smash{false}}}%
            \put(0.19994174,0.08691667){\color[rgb]{0,0,0}\makebox(0,0)[lb]{\smash{true}}}%
            \put(0.52010617,0.7009468){\color[rgb]{0,0,0}\makebox(0,0)[lb]{\smash{true}}}%
            \put(0.52010617,0.84067467){\color[rgb]{0,0,0}\makebox(0,0)[lb]{\smash{false}}}%
            \put(0.17545611,0.97192142){\color[rgb]{0,0,0}\makebox(0,0)[lb]{\smash{$\position_\current \leftarrow \position_\start$}}}%
            \put(0.20683642,0.94006345){\color[rgb]{0,0,0}\makebox(0,0)[lb]{\smash{$\position_\final$}}}%
            \put(0.25636349,0.64430486){\color[rgb]{0,0,0}\makebox(0,0)[lb]{\smash{$\tilde{\waypoints}[\disc]$}}}%
    }
    \end{picture}%
\endgroup%

\caption{Flow diagram of the motion planning algorithm}
\label{fig:path_planning_diagram}
\end{figure}

\subsubsection{RRT-Connect}

An RRT-Connect algorithm is used to generate a sequence of points inside $\navigationspace[\disc]$ connecting the initial and final positions. It receives as input the field of view $\frustum[\disc]$ of the robot, the initial and final positions and $\navigationspace[\disc]$ defined by the map representation and uses two separate trees, $\Tree_\current$ and $\Tree_\final$ initialized at the current position $\position_\current$ and final position $\position_\final$, respectively. Both trees are extended simultaneously as an RRT*~\cite{Karaman2011Sampling-basedPlanning}, i.e. a random point is obtained from $\volume$ and rejected if it is not in $\navigationspace[\disc]$, otherwise an attempt is made to connect it with the tree while minimizing a cost function. In this work, we use the Euclidean distance as cost function.

Having a tree sampling inside $\frustum[\disc]$ ensures that the navigation is always inside $\frustum[\disc]$. To see why this is of paramount importance consider a situation where the goal position is above the starting position with an obstacle in-between; the robot is a quadcopter with a front-facing camera. For this scenario, a naive approach, based on a single RRT*, would create a path in a straight line from the starting position to the goal, leading toward the obstacle above the robot, but since the trajectory for the pitch is a function of the $x$, $y$, $z$, and $\yaw$ trajectories, then the robot will not discover the obstacle before a collision. In our approach, the robot wouldn't be allowed to fly directly towards the goal. Instead it would have been forced to navigate inside $\frustum[\disc]$, discovering the obstacle in the process. 

The $\Tree_\current$ tree has the peculiarity that the sampling region to generate random points is confined to $\freespace[\disc]$. Nodes in $\Tree_\current$ are ordered by distance to the robot, such that farthest nodes are checked first when trying to connect $\Tree_\final$ with $\Tree_\current$. This critical feature forces the robot to always move in its field of view even in a previously explored area, which is essential due to uncertainty introduced by the sensors.
%Additionally, navigating in $\freespace[\disc]$ favors paths to continue in the direction of the current motion. 
%This is an important feature because as pointed out in~\cite{Cieslewski2017RapidFlight} stop and go is energetically more expensive than coasting at uniform speed.
%Additionally, navigating in $\freespace[\disc]$ favors paths to continue in the direction of the current motion, reducing the number of stops or changes in direction. This is an important feature because as pointed out in~\cite{Cieslewski2017RapidFlight} stop and go is energetically more expensive than coasting at uniform speed.

%Afterwards, for each node connected in one tree, if at least one connection can be made with the other tree while being in X allowed, the path is obtained. If more than one connection can be made, the connection with the lowest cost is selected. 

Figure \ref{subfig:first_step} illustrates both trees, $\Tree_\final$ with blue lines and nodes, and $\Tree_\current$ with green lines and nodes lying in $\frustum[\disc]$, which is represented by the purple polygon. Finally the solution is shown with red lines.

\subsubsection{Collision test during planning} Besides point cloud processing, collision detection during planning is the bottle neck in our framework. The basic operation of RRT-Connect algorithm is the random sampling of the map space in search of new  points that can be joined to an existing node in the search tree via a straight line, without intersecting occupied voxels in the map. Our algorithm for collision detection takes advantage of the octree structure for efficiency and it is presented in Algorithm \ref{alg:collision}. It receives the \console{line} to be collision tested and the \console{key} of a voxel containing both ends of the \console{line}. This algorithm tests recursively for collisions in all levels starting from \console{key.level} up to $\maxlevel$.

%If the voxel is at the maximum depth of the octree, it is sufficient to test if it is intersected by the line and if it is not in $\navigationspace[k]$. 

\begin{algorithm}
\label{alg:collision}
\small{
\SetAlgoLined
%\KwResult{collision}

  \eIf{\console{key.level} == $\maxlevel$}
  {
        $v\leftarrow$  \console{getVoxel(key)}
    
        \eIf{\console{line} intersects with $v$ {\bf and} $v\nsubseteq\navigationspace[k]$}
        {
            {\bf return} true
        }
        {
            {\bf return} false
        }
   
   }{
        \console{children}$\leftarrow$\console{getChildren(key)}
    
        \For{{\bf each} \console{child} {\bf in} \console{children}}
        {
            \If{\console{detectCollision(child,line)}}
            {
                {\bf return} true   
            }
        }
        {\bf return} false
  }
}
 \caption{\console{detectCollision(key,line)}}
\end{algorithm}

Fig.~\ref{fig:lines_comparison} presents a benchmark comparison between Algorithm~\ref{alg:collision}, which exploits our map representation, and the naive way of collision checking on the same map by testing every voxel between the two points defining a line. Fig.~\ref{fig:lines_comparison} shows the average computing time for 10000 lines with lengths between 5 to 400 voxels in the same random forest environment described in Section \ref{sec:results}. In contrast with the naive method, which has a computation time that grows close to linearly with the size of the line, the computation time remains nearly flat with our map representation and Algorithm~\ref{alg:collision} as it mainly depends on the depth $\maplevel$ of the linear octree. Note that the expected length between two random points in this map is 165 voxels.

\begin{figure}
\centering
\begin{tikzpicture}[scale=0.8, transform shape]

    \begin{axis}[
        xlabel= $\text{line length in voxels } $,
        ylabel=$\text{microseconds [}\mu\text{s] }$,
        legend pos=north west]
    \addplot[red] plot coordinates {
       ( 1.4434/0.2 , 49.977) 
( 2.8868/0.2 , 41.555) 
( 4.3301/0.2 , 43.171) 
( 5.7735/0.2 , 52.098) 
( 7.2169/0.2 , 62.576) 
( 8.6603/0.2 , 59.147) 
( 10.104/0.2 , 72.371) 
( 11.547/0.2 , 71.009) 
( 12.99/0.2 , 72.225) 
( 14.434/0.2 , 65.798) 
( 15.877/0.2 , 64.441) 
( 17.321/0.2 , 59.182) 
( 18.764/0.2 , 60.18) 
( 20.207/0.2 , 66.193) 
( 21.651/0.2 , 67.13) 
( 23.094/0.2 , 66.35) 
( 24.537/0.2 , 57.826) 
( 25.981/0.2 , 70.616) 
( 27.424/0.2 , 74.059) 
( 28.868/0.2 , 83.64) 
( 30.311/0.2 , 66.617) 
( 31.754/0.2 , 71.343) 
( 33.198/0.2 , 75.607) 
( 34.641/0.2 , 71.604) 
( 36.084/0.2 , 78.323) 
( 37.528/0.2 , 67.079) 
( 38.971/0.2 , 64.269) 
( 40.415/0.2 , 84.904) 
( 41.858/0.2 , 73.476) 
( 43.301/0.2 , 83.05) 
( 44.745/0.2 , 75.463) 
( 46.188/0.2 , 56.009) 
( 47.631/0.2 , 52.513) 
( 49.075/0.2 , 97.188) 
( 50.518/0.2 , 79.136) 
( 51.962/0.2 , 86.313) 
( 53.405/0.2 , 81.637) 
( 54.848/0.2 , 83.735) 
( 56.292/0.2 , 91.259) 
( 57.735/0.2 , 83.654) 
( 59.178/0.2 , 69.333) 
( 60.622/0.2 , 73.704) 
( 62.065/0.2 , 85.275) 
( 63.509/0.2 , 81.058) 
( 64.952/0.2 , 78.411) 
( 66.395/0.2 , 90.007) 
( 67.839/0.2 , 82.338) 
( 69.282/0.2 , 81.465) 
( 70.725/0.2 , 86.399) 
( 72.169/0.2 , 93.241) 
( 73.612/0.2 , 75.93) 
( 75.056/0.2 , 72.295) 
( 76.499/0.2 , 74.265) 
( 77.942/0.2 , 97.539) 
( 79.386/0.2 , 84.184) 
( 80.829/0.2 , 84.697) 
( 82.272/0.2 , 79.45) 
( 83.716/0.2 , 60.815) 
( 85.159/0.2 , 70.765) 
( 86.603/0.2 , 79.273) 
    };
    \addlegendentry{detectCollision}

     \addplot[color=blue]
         plot coordinates {
             ( 1.4434/0.2 , 0.48228) 
( 2.8868/0.2 , 7.5775) 
( 4.3301/0.2 , 14.154) 
( 5.7735/0.2 , 22.295) 
( 7.2169/0.2 , 34.896) 
( 8.6603/0.2 , 35.779) 
( 10.104/0.2 , 43.34) 
( 11.547/0.2 , 50.548) 
( 12.99/0.2 , 53.913) 
( 14.434/0.2 , 51.509) 
( 15.877/0.2 , 60.812) 
( 17.321/0.2 , 56.408) 
( 18.764/0.2 , 58.312) 
( 20.207/0.2 , 68.59) 
( 21.651/0.2 , 75.817) 
( 23.094/0.2 , 72.14) 
( 24.537/0.2 , 67.212) 
( 25.981/0.2 , 82.272) 
( 27.424/0.2 , 88.691) 
( 28.868/0.2 , 101.624) 
( 30.311/0.2 , 81.438) 
( 31.754/0.2 , 91.438) 
( 33.198/0.2 , 99.618) 
( 34.641/0.2 , 101.321) 
( 36.084/0.2 , 107.197) 
( 37.528/0.2 , 96.85) 
( 38.971/0.2 , 98.271) 
( 40.415/0.2 , 129.142) 
( 41.858/0.2 , 117.178) 
( 43.301/0.2 , 155.101) 
( 44.745/0.2 , 123.113) 
( 46.188/0.2 , 94.849) 
( 47.631/0.2 , 93.734) 
( 49.075/0.2 , 169.581) 
( 50.518/0.2 , 142.855) 
( 51.962/0.2 , 151.201) 
( 53.405/0.2 , 151.269) 
( 54.848/0.2 , 145.828) 
( 56.292/0.2 , 169.355) 
( 57.735/0.2 , 164.662) 
( 59.178/0.2 , 141.731) 
( 60.622/0.2 , 149.881) 
( 62.065/0.2 , 181.025) 
( 63.509/0.2 , 166.778) 
( 64.952/0.2 , 172.153) 
( 66.395/0.2 , 187.115) 
( 67.839/0.2 , 172.621) 
( 69.282/0.2 , 176.828) 
( 70.725/0.2 , 196.824) 
( 72.169/0.2 , 189.808) 
( 73.612/0.2 , 162.185) 
( 75.056/0.2 , 163.643) 
( 76.499/0.2 , 177.065) 
( 77.942/0.2 , 218.144) 
( 79.386/0.2 , 199.282) 
( 80.829/0.2 , 211.893) 
( 82.272/0.2 , 184.516) 
( 83.716/0.2 , 153.933) 
( 85.159/0.2 , 176.651) 
( 86.603/0.2 , 194.622) 
         };
     \addlegendentry{discretization}
       
    \end{axis}
 
\end{tikzpicture}
 \caption{Benchmark of collision detection using Algorithm~\ref{alg:collision} against a discretization-based collision detection}
\label{fig:lines_comparison}
\end{figure}
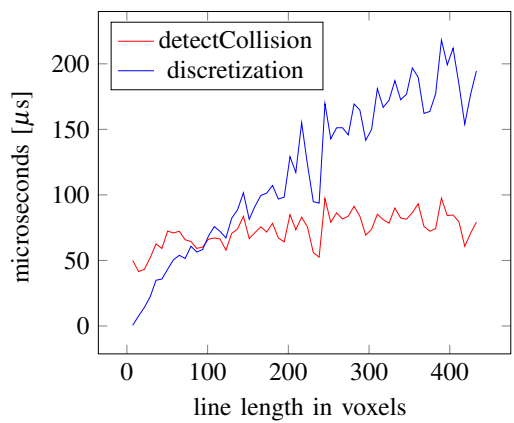

%Its complexity is logarithmic on the size of the smallest voxel containing the start and end point of the line to be collision-checked. 
%For comparison, collision detection algorithms based on a pure spatial hashing approach with constant look-up time have complexity proportional to the length of the test line, with proportionality constant related to the maximum cross section dimension of the robot. For a spherical robot with diameter $d$ measured in voxels, the number of queries to the hash table would be in the order of $d^2$ times the length of the line in voxels. To get an idea of how many queries could be expected on average, the cube line picking length can be computed using the formula known as the Robbins constant~\cite{} [Phillip'91]. For a cubic map of length 12.8m, the average line is 8.47m long. Considering a robot with dimension 3 voxels of size 20cm, close to 400 queries would be needed to ensure that the line was collision free. With an octree, assuming the line was collision free at least half the voxels along the line would need to be of size 40cm, giving at least half the number of collision checks but likely much fewer in a sparsely populated map. 
 
\subsection{Trajectory Generation} \label{SecGen}

The solution of the path planning algorithm is used for the generation of dynamically feasible trajectories. The proposed trajectory generation algorithm is described in this section. Given that the model of the UAV is differentially flat, the dynamics for each component of the position and yaw are considered to be decoupled fourth and second order integrator chains, respectively~\cite{Mellinger2011}.

The given path is specified in terms of the sets of waypoints $\tilde{\waypoints}[\disc]$ that the path planning has generated and pushed to a stack $H_K = \{h_1,\dots,h_K\}$ where $h_i\in\mathbb{R}^4$ consist of the $i$-th position and orientation point, as illustrated in Figure \ref{fig:path_planning_diagram}. In this section, we propose an algorithm to traverse the path while navigating in the region $\sallowed[\disc]$ given in~\eqref{Eq:Zsafe}. The trajectory generation guarantees that the dynamic constraints are satisfied while allowing an a priori defined maximum separation $E_p$ from the path (notice that the regions definition given in Figure~\ref{fig:regions} takes into account $E_p$) and a maximum deviation $E_\yaw$ from the desired orientation $\yaw_i$ at any waypoint in the trajectory. The latter constraint ensures that during navigation the next waypoint is always inside the field of view of the robot from the last waypoint reached. The former constraint guarantees that the position $\position(\timme)$ remains inside $\mathcal{V}_{\text{safe}}[\disc]$, as illustrated in Fig.~\ref{fig:image_filter}, where $k = \lfloor (t~-~t_i) \framerate \rfloor$.

%The trajectory generation must satisfy a maximum separation from the from the path of $E_p$, which is given a priori (notice that the regions definition given in Figure~\ref{fig:regions} take into account $E_p$). In a similar way, a deviation $E_\yaw$ of less than half of the field of view, between the real orientation and the desired orientation $\yaw_i$ at a given waypoint, is allowed, since still ensures that the next waypoint is in the field of view.

Our method consists of two components. The first one is a virtual control (for the snap of each of the spatial coordinates and the yaw's angular acceleration) %$u=-Ke$, $e(t) = z(t) -z^*$ 
that induces an asymptotically stable equilibrium point $z^*=[\eta^T,0,\dots, 0]^T\in\mathbb{R}^{14}$ which has an associated basin of attraction around $z^*$. %Hence, the goal of the controller to drive the error $e(t) = z(t) - z^*$ towards the origin. %Thus, to each $z_0\in\mathbb{R}^{14}$ corresponds a region $\ellipse(z_0)$ around $z^*$ such that a trajectory $z(t;z_0,t_0)$ passing through $z_0$ at $t=t_0$ cannot leave for $t\geq t_0$. 
In particular, given the requirement of maintaining the trajectory $z(t)$ inside $\sallowed[\disc]$, we can find the set $\ellipse(\timme)\subset\sallowed[\disc]$ of points around $z^*$ satisfying this requirement, such that the trajectory $z(t)$ passing through $z_0\in\ellipse(\timme)$ at $t=t_0$ cannot leave $\ellipse(\timme)$ for $t\geq t_0$. Moreover, the possible values for the equilibrium points are restricted to lie in the line segments connecting consecutive waypoints in $H_K$. Thus, $\eta$ is parametrized by $s$ where $s\in\mathbb{R}$ and $\eta(S)$ is the end of the path. 

The second component of the trajectory generation provides a dynamic for $s(t)$ to evolve $\eta(s(t))$  (thus moving the equilibrium point $z^*$) along the path defined by $H_K$, such that $\dot{s}(t)>0$ (thus, moving forward along the path) as long as the trajectory $z(t)$ is in the interior of $\ellipse(\timme)$, and $\dot{s}(t)$ goes to zero as $z(t)$ approaches the boundary of $\ellipse(\timme)$. Assuming that at the beginning, $z(0)=[\eta(s(0))^T,0,\dots,0]\in\ellipse(\timme)$ then, $s(t)$ is a monotonically increasing function and evolves guaranteeing that $z(t)\in\ellipse(\timme)\subset\sallowed[\disc]$, $\forall t\geq 0$.

\begin{figure}
    \centering
    %\includegraphics[scale=0.7]{images/path_diagram.png}
    %\vspace{-0.7in}
    \def\svgwidth{8.5cm}
\begingroup%
  \makeatletter%
  \providecommand\color[2][]{%
    \errmessage{(Inkscape) Color is used for the text in Inkscape, but the package 'color.sty' is not loaded}%
    \renewcommand\color[2][]{}%
  }%
  \providecommand\transparent[1]{%
    \errmessage{(Inkscape) Transparency is used (non-zero) for the text in Inkscape, but the package 'transparent.sty' is not loaded}%
    \renewcommand\transparent[1]{}%
  }%
  \providecommand\rotatebox[2]{#2}%
  \ifx\svgwidth\undefined%
    \setlength{\unitlength}{245.02282604bp}%
    \ifx\svgscale\undefined%
      \relax%
    \else%
      \setlength{\unitlength}{\unitlength * \real{\svgscale}}%
    \fi%
  \else%
    \setlength{\unitlength}{\svgwidth}%
  \fi%
  \global\let\svgwidth\undefined%
  \global\let\svgscale\undefined%
  \makeatother%
  \begin{picture}(1,0.68206692)%  
      \put(0,0){\includegraphics[width=\unitlength,page=1]{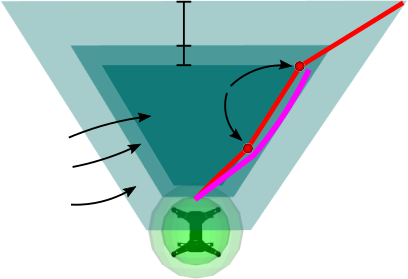}}%
      %\put(0,0){\includegraphics[width=\unitlength,page=2]{images/filter_example.pdf}}%
      %\put(0,0){\includegraphics[width=\unitlength,page=3]{images/filter_example.pdf}}%
      
      \normalsize{
                \put(0.46795721,0.45639467){\color[rgb]{0,0,0}\makebox(0,0)[lb]{\smash{$\tilde{\waypoints}[\disc]$}}}%
                \put(0.69055924,0.33736307){\color[rgb]{0,0,0}\makebox(0,0)[lb]{\smash{$\position(\timme)$}}}%
               
                \put(0.37438041,0.61613427){\color[rgb]{0,0,0}\makebox(0,0)[lb]{\smash{$\text{r}$}}}%
                \put(0.37405137,0.535269884){\color[rgb]{0,0,0}\makebox(0,0)[lb]{\smash{$E_p$}}}%
                
                \put(0.05,0.330379){\color[rgb]{0,0,0}\makebox(0,0)[lb]{\smash{$\freespace[\disc]$}}}%
                \put(0.05,0.25735888){\color[rgb]{0,0,0}\makebox(0,0)[lb]{\smash{$\mathcal{V}_{\text{safe}}[\disc]$}}}%
                \put(0.05,0.17167986){\color[rgb]{0,0,0}\makebox(0,0)[lb]{\smash{$\castspace[\disc]$}}}%

    }
    \end{picture}%
\endgroup%
    \caption{Illustration of trajectory generation allowing a deviation of $E_p$ from the path}
    \label{fig:image_filter}
\end{figure}
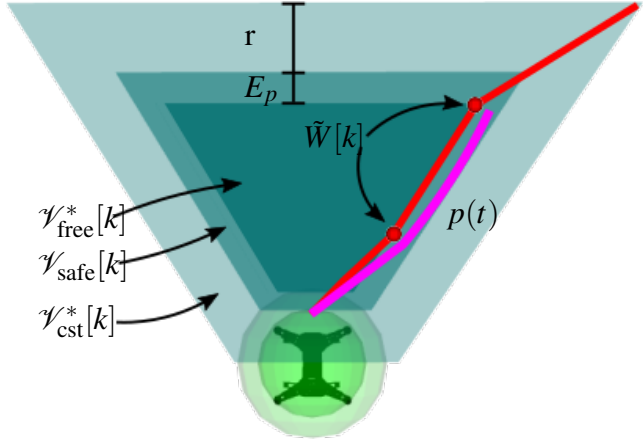

\begin{table*}[]
\centering
\begin{tabular}{c|ccccccc}
\hline
& \textbf{\begin{tabular}[c]{@{}c@{}}Oleynikova\\  et. al. \cite{Oleynikova2018SafeVehicles}\end{tabular}} & \textbf{\begin{tabular}[c]{@{}c@{}}Gao\\ et.  al. \cite{Gao2018OnlinePolynomial}\end{tabular}}  & \textbf{\begin{tabular}[c]{@{}c@{}}Florence\\ et.  al. \cite{Florence2018NanoMap:Data}\end{tabular}} & \textbf{\begin{tabular}[c]{@{}c@{}}Usenko\\ et.  al. \cite{Usenko2017Real-TimeBuffer}\end{tabular}} & \textbf{\begin{tabular}[c]{@{}c@{}}Mohta\\ et al. \cite{Mohta2018FastEnvironments}\end{tabular}}          & \textbf{\begin{tabular}[c]{@{}c@{}}Lin\\ et al.\cite{Lin2018AutonomousFusion} \end{tabular}}       & \textbf{Ours}                                                    \\ \hline
\textbf{\begin{tabular}[c]{@{}c@{}}Navigation\\ Objectives\end{tabular}}          & Goal                                                                   & Goal       & Goal                                                                 & Goal                                                               & Goal                                                                     & Goal                                                                & Goal                                                             \\
\rowcolor[HTML]{EFEFEF} 
\textbf{\begin{tabular}[c]{@{}c@{}}Navigation\\ in  FoV\end{tabular}}             & Yes                                                                    & No                                                               & Yes                                                                  & No                                                                 & No                                                                       & No                                                                  & Yes                                                              \\
\textbf{\begin{tabular}[c]{@{}c@{}}Feasible\\ Yaw\\ Dynamics\end{tabular}}        & No                                                                     & No                                                               & Yes                                                                  & No                                                                 & No                                                                       & No                                                                  & Yes                                                              \\
\rowcolor[HTML]{EFEFEF} 
\textbf{\begin{tabular}[c]{@{}c@{}}Feasible\\ Positional\\ Dynamics\end{tabular}} & Yes                                                                    & Yes                                                              & Yes                                                                  & Yes                                                                & Yes                                                                      & No                                                                 & Yes                                                              \\
\textbf{Map  Type}                                                                & \begin{tabular}[c]{@{}c@{}}Voxel Hashing\\ TSDF \& ESDF\end{tabular}   & \begin{tabular}[c]{@{}c@{}}Regular ESDF\\ Grid Map\end{tabular}  & \begin{tabular}[c]{@{}c@{}}Search over\\ views\end{tabular}          & \begin{tabular}[c]{@{}c@{}}Egocentric\\ Grid\end{tabular}          & \begin{tabular}[c]{@{}c@{}}3D local map and\\ 2D global map\end{tabular} & TSDF                                                                & Linear Octree                                                    \\
\rowcolor[HTML]{EFEFEF} 
\textbf{\begin{tabular}[c]{@{}c@{}}Escaping\\ Pockets\end{tabular}}               & Yes                                                                    & Yes                                                              & No                                                                   & No                                                                 & Yes                                                                      & Yes                                                                 & Yes                                                              \\
\textbf{\begin{tabular}[c]{@{}c@{}}Hardware\\ Characteristics\end{tabular}}       & \begin{tabular}[c]{@{}c@{}}2.4 GHz\\ i7 dual-core\end{tabular}         & \begin{tabular}[c]{@{}c@{}}3.00 GHz i7\\ Nvidia TX1\end{tabular} & i7 dual-core                                                         & \begin{tabular}[c]{@{}c@{}}2.1 GHz\\ i7 dual-core\end{tabular}     & \begin{tabular}[c]{@{}c@{}}3.40 GHz\\ Intel i7-5557U\end{tabular}        & \begin{tabular}[c]{@{}c@{}}Intel i7-5500U\\ Nvidia TX1\end{tabular} & \begin{tabular}[c]{@{}c@{}}1.6 GHz\\ Atom x7-Z8750\end{tabular}  \\
\rowcolor[HTML]{EFEFEF} 
\textbf{\begin{tabular}[c]{@{}c@{}}Operation\\ Frequency\end{tabular}}            & 4Hz                                                                    & 10Hz                                                             & 200-250Hz                                                            & N/A                                                                & 3Hz                                                                      & 4 Hz                                                                & 33Hz                                                             \\
\textbf{\begin{tabular}[c]{@{}c@{}}Code\\ Available\end{tabular}}                 & Not Complete                                                           & Yes                                                              & Not Complete                                                         & Yes                                                                & No                                                                       & No                                                                  & \begin{tabular}[c]{@{}c@{}}Will be made\\ available\end{tabular} \\ \hline
\end{tabular}
\caption{Qualitative Comparison of state-of-the-art methods for navigation in Unknown Environments. Continues on Table~\ref{tab:comparisonII}}
\label{tab:comparison}
\end{table*}

%%% and a controller which moves the attractor along the path while simultaneously guaranteeing that 1) the trajectory remains inside the attraction set and 2) that progress is always made towards the next way point.

\begin{theorem}
Let $\dot{\position}(\timme)=\velocity(\timme)$, $\dot{\velocity}(\timme)=\acceleration(\timme)$, $\dot{\acceleration}(\timme) = \jerk(\timme)$, $\dot{j}(\timme) = \snap(\timme)$, $\dot{\yaw}(\timme)=\angvelocity(\timme)$, $\dot{\angvelocity}(\timme) = \angacceleration(\timme)$ be the dynamics of the center point of a robot. Let $\eta(s(t))=[\eta_x(s(t)),\eta_y(s(t)),\eta_z(s(t)),\eta_\yaw(s(t))]^T$ represent the parametrization of the line segments connecting consecutive points in $H_K$ and let $z^*(t)=[\eta(s(t)),0,\dots,0]^T$. Therefore, if $z(0)=z^*(0)$ then for the given $E_p$ and $E_\yaw$,  by using
\begin{align}
    \snap(t) =& k_1(\position-[\eta_x(s(t)),\eta_y(s(t)),\eta_z(s(t))]^T)\\
    &+k_2\velocity(t)+k_3\acceleration(t)+k_4\jerk(t)\\
    \angacceleration(t) =& k_5(\yaw(t) - \eta_\yaw(s(t)))+k_6\angvelocity(t)\\
    \dot{s}(t) =& \max\{\rho-(z(t)-z^*(t))^TP(z(t)-z^*(t)) ,0\}
\end{align}
where  $z(t) = [p(t)^T,\yaw(t),v(t)^T,\omega(t),a(t)^T,j(t)^T]^T$; $A$ is such that with the above definition of $\snap$ and $\angacceleration$, $\dot{z} = Az\big\rvert_{\eta(s)=0}$ holds; $P$ satisfies $PA+A^TP=-I$; and $\rho>0$ is chosen such that the hyper-ellipse $\mathcal{E}(t)=\{z:(z-z^*(t))^TP(z-z^*(t))\leq\rho\}$ lies inside the set
\begin{align}
\mathcal{Z}_{\eta}(t)=\{\state\ :&|\yaw(t)-\eta_\yaw(s(t))|\leq E_\yaw,\\
&\|p(t)-[\eta_x(s(t)),\eta_y(s(t)),\eta_z(s(t))]^T\|_\infty\leq E_p, \\
&\|v(t)\|_\infty\leq\vmax,\  
\|a(t)\|_\infty\leq\amax,\\ 
&\|j(t)\|_\infty\leq\jmax,\
|\omega(t)|\leq\omegamax\}, 
\end{align}
$z(t)$ is maintained inside $\mathcal{Z}_{\eta}(t)\subset\sallowed[\disc]$ for all $t\geq 0$ and asymptotically reaches $[\eta(S),0,\dots,0]^T$. 
\end{theorem} 
\begin{proof}
Notice that $z(0)\in\mathcal{E}(0)$, moreover while $z(t)\in\mathcal{E}(t)$, $\dot{s}\geq0$ and $\eta(s)$ progresses along the path. Let $e(t)=z(t)-z^*(t)$ with $\dot{e} = Ae - \dot{z}^*$. Consider the Lyapunov function candidate $V = e^TPe$. Therefore, $\dot{V} = e^T(PA+A^TP)e + 2e^TP\dot{z}^*= -e^Te + 2e^TP\dot{z}^*$ and the value of $V$ decreases for $2e^TP\dot{z}^*(t)\leq 0$. 

Suppose that, due to the evolution of $s(t)$ and as a consequence the evolution of $z^*(t)$, $z(t)$ evolves from $z(t_1)\in\mathcal{E}(t_1)$ to $z(t_2)$ in the boundary of $\mathcal{E}(t_2)$, i.e. where $e(t_2)^TPe(t_2)=\rho$, with $e(t)^TPe(t)\leq\rho, \forall t\in[t_1,t_2]$. This implies that, as $z(t)$ approaches the boundary of $\mathcal{E}$ the velocity of $s(t)$ decreases, until $\dot{s}(t_2)=0$. Thus, $2e^T(t_2)P\dot{z}^*(t_2) = 2e^T(t_2)P\left[\frac{d}{ds}\eta(s)\big\rvert_{s=s(t_2)},0,\dots,0\right]\dot{s}(t_2) = 0$. Hence, $V$ is decreasing at $t=t_2$ and $z(t)$ will move toward the interior of $\mathcal{E}(t)$ where $\dot{s}(t)>0$. Thus, $\eta(s)$ progresses along the path. Therefore, the condition $V<\rho$ will remain. If eventually, for a time $t_3$,  $e^T(t_3)Pe^T(t_3)=\rho$, then the process is repeated and $\eta(s(t))$ progresses along the path satisfying $z(t)\in\mathcal{E}(t)$. Eventually, $s$ will be such that $\eta(s(t))$ is at the end of the path and remains there i.e. $\eta(s(t))$ remain constant, under this scenario $\dot{V} = e^T(PA+A^TP)e = -e^Te<0$ for all $z\neq [\eta(S),0,\dots,0]^T$. Thus, $z(t)$ will approach $[\eta(S),0,\dots,0]^T$ asymptotically.
\end{proof}

This method is similar to contouring control~\cite{Liniger2015Optimization-basedCars} in that the path is traversed by continuously making a trade-off between separation from the path and velocity. However, unlike existing contouring methods, our approach doesn't require solving an optimization problem online. It is also similar to LQR-Trees~\cite{Tedrake2010LQR-trees:Verification} in that Lyapunov analysis is used to induce ``funnels" along the path. However, our method is significantly simpler since LQR-Trees are computationally equivalent to Kinodynamic planning in $\mathbb{R}^{14}$. Our method on the other hand exploits the fast exploration of RRTs in $\mathbb{R}^3$ to find an obstacle-free path online. This is coupled with an offline stage of LQR virtual control design and Lyapunov analysis to compute the basin of attraction. Thus, given a path, to generate a dynamically feasible trajectory, in the online stage, we only require to evolve the chain of integrators of position and yaw together with evolving $s(t)$, to move the equilibrium point, progressing along the path.

\section{Discussion}

In this section, we present a qualitative comparison, summarized in Table~\ref{tab:comparison} and Table~\ref{tab:comparisonII}, with the state of the art methods. In particular, we focus this discussion in the contributions reported in~\cite{Oleynikova2018SafeVehicles,Gao2018OnlinePolynomial,Cieslewski2017RapidFlight,Papachristos2017Uncertainty-awareRobots,Florence2018NanoMap:Data,Usenko2017Real-TimeBuffer,Selin2019EfficientEnvironments,Mohta2018FastEnvironments,Lin2018AutonomousFusion}, as in our opinion, they represent the closest approaches.
These contributions can be divided into two problems, reaching a goal in an unknown environment~\cite{Oleynikova2018SafeVehicles,Gao2018OnlinePolynomial,Florence2018NanoMap:Data,Usenko2017Real-TimeBuffer,Mohta2018FastEnvironments,Lin2018AutonomousFusion} summarized in Table~\ref{tab:comparison} and exploration of an unknown environment~\cite{Cieslewski2017RapidFlight,Papachristos2017Uncertainty-awareRobots,Selin2019EfficientEnvironments}, summarized in Table~\ref{tab:comparisonII}. The former is the problem addressed in this paper, but since exploration require similar components contributions to both problems are discussed. In fact, the work of \cite{Oleynikova2018SafeVehicles} is an extension to their previous exploration framework by introducing a soft cost for reaching a goal. 

Once a goal has been determined, a critical aspect for collision-free navigation in unknown environments is imposing constraints on the motion plan to navigate within the current field of view, where the robot has the highest confidence on the traversable space. Our method incorporates this idea which has also been used by \cite{Oleynikova2018SafeVehicles,Cieslewski2017RapidFlight,Florence2018NanoMap:Data}. The work of \cite{Gao2018OnlinePolynomial} assumes that the sensors onboard the robot cover an entire spherical region centered at the robot, which is hard to achieve in practice.

Another important aspect related to navigation in the field of view is ensuring that the camera axis, and hence the yaw orientation is aligned with the direction of motion. The work in \cite{Oleynikova2018SafeVehicles,Cieslewski2017RapidFlight} attempts to accomplish this restriction by implementing a velocity tracking yaw approach but they don't guarantee that the generated yaw trajectories will meet angular dynamical constraints.  

Additionally, positional dynamical constraints of the robot must be handled explicitly to guarantee that it will remain in safe regions. In particular, the work in \cite{Cieslewski2017RapidFlight,Selin2019EfficientEnvironments,Lin2018AutonomousFusion} does not handle dynamical constraints explicitly. Moreover,  \cite{Usenko2017Real-TimeBuffer} solve an optimization problem with soft limits on the time derivatives of the position over the trajectory. Finally, \cite{Oleynikova2018SafeVehicles,Gao2018OnlinePolynomial,Papachristos2017Uncertainty-awareRobots,Florence2018NanoMap:Data,Mohta2018FastEnvironments} solve a constrained optimization problem to handle positional dynamical constraints. However, solving optimization problems online is computationally expensive, even in the convex case. In our approach, the dynamical constraints are incorporated without the need of an optimization-based solution.

\begin{table}[]
\centering
\begin{tabular}{c|ccc}
\hline
& \textbf{\begin{tabular}[c]{@{}c@{}}Cieslewski\\ et. al \cite{Cieslewski2017RapidFlightc}\end{tabular}} & \textbf{\begin{tabular}[c]{@{}c@{}}Papachristos\\ et. al. \cite{Papachristos2017Uncertainty-awareRobots}\end{tabular}} & \textbf{\begin{tabular}[c]{@{}c@{}}Selin\\ et. al. \cite{Selin2019EfficientEnvironments}\end{tabular}} \\ \hline
\textbf{\begin{tabular}[c]{@{}c@{}}Navigation\\ Objectives\end{tabular}}          & Exploration                                                            & Exploration                                                              & Exploration                                                       \\
\rowcolor[HTML]{EFEFEF} 
\textbf{\begin{tabular}[c]{@{}c@{}}Navigation\\ in  FoV\end{tabular}}             & Yes                                                                    & No                                                                       & No                                                                \\
\textbf{\begin{tabular}[c]{@{}c@{}}Feasible\\ Yaw\\ Dynamics\end{tabular}}        & No                                                                     & Yes                                                                      & No                                                                \\
\rowcolor[HTML]{EFEFEF} 
\textbf{\begin{tabular}[c]{@{}c@{}}Feasible\\ Positional\\ Dynamics\end{tabular}} & No                                                                     & No                                                                       & No                                                                \\
\textbf{Map  Type}                                                                & Octomap                                                                & Octomap                                                                  & Octomap                                                           \\
\rowcolor[HTML]{EFEFEF} 
\textbf{\begin{tabular}[c]{@{}c@{}}Escaping\\ Pockets\end{tabular}}               & Yes                                                                    & No                                                                       & Yes                                                               \\
\textbf{\begin{tabular}[c]{@{}c@{}}Hardware\\ Characteristics\end{tabular}}       & N/A                                                                    & \begin{tabular}[c]{@{}c@{}}3.4 GHz\\ i7 dual-core\end{tabular}           & N/A                                                               \\
\rowcolor[HTML]{EFEFEF} 
\textbf{\begin{tabular}[c]{@{}c@{}}Operation\\ Frequency\end{tabular}}            & N/A                                                                    & N/A                                                                      & N/A                                                               \\
\textbf{\begin{tabular}[c]{@{}c@{}}Code\\ Available\end{tabular}}                 & No                                                                     & Yes                                                                      & Yes                                                               \\ \hline
\end{tabular}
\caption{Continuation of Table~\ref{tab:comparison}: Qualitative Comparison of state-of-the-art methods for navigation in Unknown Environments}
\label{tab:comparisonII}
\end{table}

The ability to escape ``pockets" or getting out of dead ends is fundamental to complete the navigation task in general cluttered environments. Beyond local collision avoidance, this requires maintaining and keeping an up to date map of the explored areas together with a strategy for handling unexplored regions. Based on this map the robot should be able to generate a motion plan from the current position to the goal. Of the methods reviewed,~\cite{Papachristos2017Uncertainty-awareRobots} is purely local while~\cite{Usenko2017Real-TimeBuffer} uses an egocentric grid of fixed size to create a local map representation around the robot. ~\cite{Florence2018NanoMap:Data, Mohta2018FastEnvironments} use global planning based on A* on a 2D map to guide local exploration. Compared to our solution, the methods in ~\cite{Oleynikova2018AEnvironments,Gao2018OnlinePolynomial,Lin2018AutonomousFusion} are computationally expensive due to the need to generate signed distance fields. As argued in section \ref{SecMapping}, the map type and implementation have a significant impact not only on planning global trajectories but on the overall efficiency of the complete navigation framework.

Regarding the hardware characteristics in which the method was implemented, we can highlight that our approach was demonstrated to work with the lowest processor requirements, while~\cite{Gao2018FlyingEnvironments,Oleynikova2018SafeVehicles,Papachristos2017Uncertainty-awareRobots,Florence2018NanoMap:Data,Usenko2017Real-TimeBuffer,Mohta2018FastEnvironments, Lin2018AutonomousFusion} used high end Intel~ i7 processors and~\cite{Gao2018FlyingEnvironments, Lin2018AutonomousFusion} used an additional Tx1 Nvidia GPU. Moreover, as detailed in the next section we were able to plan and replan trajectories at the camera frame rate ($33\mathrm{Hz}$), while~\cite{Oleynikova2018SafeVehicles} and~\cite{Gao2018FlyingEnvironments} reported an operating frequency of $4\mathrm{Hz}$ and $10\mathrm{Hz}$, respectively. To the best of our knowledge, the platform used in our experiments represents the most frugal platform for autonomous navigation in unknown cluttered environments demonstrated to date.

In the next section, quantitative comparisons with state of the art methods are made. In particular, we focus on the contributions addressing the reaching a goal in an unknown environment problem~\cite{Oleynikova2018SafeVehicles,Gao2018OnlinePolynomial,Florence2018NanoMap:Data,Usenko2017Real-TimeBuffer,Lin2018AutonomousFusion,Mohta2018FastEnvironments}. Unfortunately, from the contributions reviewed in Table~\ref{tab:comparison} focusing on this problem, only ~\cite{Gao2018OnlinePolynomial,Usenko2017Real-TimeBuffer} have an available code. For this reason, we limit our comparison to those works, even though~\cite{Usenko2017Real-TimeBuffer} focuses on a narrowed problem which is based in the solution of a local replanning problem, which assumes the existence of a goal map and a preplanned trajectory.

% \begin{table*}
%     \centering
%     \begin{tabular}{|p{2.8cm}|p{1.8cm}|p{1.8cm}|p{1.8cm}|p{1.3cm}|p{1.5cm}|p{1cm}|p{1.4cm}|}
%     \hline
%         Reference & Safe Navigation in Static Env. & Feasible trajectories for Yaw & Hardware Characteristics & Operation Frequency & Uncertainty Management & Map Type & Handles local minima\\
%         \hline
%         Oleynikova et. al. \cite{Oleynikova2017SafeVehicles} &  \\
%         \hline
%         Gao et. al. \cite{Gao2018OnlinePolynomial} & \\
%         \hline
%         Cieslewski et. al \cite{Cieslewski2017RapidFlightc} & \\
%         \hline
%         Papachristos et. al. \cite{Papachristos2017Uncertainty-awareRobots} & \\
%         \hline
%         Florence et. al. \cite{Florence2018NanoMap:Data} & \\
%         \hline
%         Bircher et. al. \cite{Bircher2016RecedingExploration} &\\
%         \hline
%         Usenko et. al. \cite{Usenko2017Real-TimeBuffer} & \\
%         \hline
%         Selin et. al. \cite{Selin2019EfficientEnvironments} & \\
%         \hline
%         Ours & Yes & Yes &\\
%         \hline
        
%     \end{tabular}
%     \caption{Comparative Table}
%     \label{tab:my_label}
% \end{table*}

%

\section{Results}\label{sec:results}

The proposed algorithm was implemented using the C++ language and the ROS Kinetic framework~\cite{Quigley2009ROS:System}. The benchmark results are shown comparing the proposed algorithm with state-of-the-art navigation algorithms in several simulated Poisson forest scenarios generated using the method in~\cite{Karaman2012High-speedForest}. The proposed algorithm was also tested in a simulated maze environment and several real-world scenarios. See the video of the experimental results at~\url{https://youtu.be/Wq0e7vF6nZM}.
%The experimental results are available at \url{www.youtube.com}.

\begin{figure*}
    \centering
%%%%%------ new figure
    \subfloat{
    \def\svgwidth{4.9cm}
\begingroup%
  \makeatletter%
  \providecommand\color[2][]{%
    \errmessage{(Inkscape) Color is used for the text in Inkscape, but the package 'color.sty' is not loaded}%
    \renewcommand\color[2][]{}%
  }%
  \providecommand\transparent[1]{%
    \errmessage{(Inkscape) Transparency is used (non-zero) for the text in Inkscape, but the package 'transparent.sty' is not loaded}%
    \renewcommand\transparent[1]{}%
  }%
  \providecommand\rotatebox[2]{#2}%
  \ifx\svgwidth\undefined%
    \setlength{\unitlength}{170.07874016bp}%
    \ifx\svgscale\undefined%
      \relax%
    \else%
      \setlength{\unitlength}{\unitlength * \real{\svgscale}}%
    \fi%
  \else%
    \setlength{\unitlength}{\svgwidth}%
  \fi%
  \global\let\svgwidth\undefined%
  \global\let\svgscale\undefined%
  \makeatother%
  \begin{picture}(0,0)%
      \put(-0.42,0){ % Created by tikzDevice version 0.12 on 2019-04-23 13:53:34
% !TEX encoding = UTF-8 Unicode
\begin{tikzpicture}[x=1pt,y=1pt]
\definecolor{fillColor}{RGB}{255,255,255}
\path[use as bounding box,fill=fillColor,fill opacity=0.00] (0,0) rectangle (411.94,187.90);
\begin{scope}
\path[clip] ( 49.20, 61.20) rectangle (386.74,138.70);
\definecolor{drawColor}{RGB}{255,0,0}

\path[draw=drawColor,line width= 1.2pt,line join=round] (148.88, 66.46) -- (148.88, 85.60);
\definecolor{drawColor}{RGB}{0,0,0}

\path[draw=drawColor,line width= 0.4pt,dash pattern=on 4pt off 4pt ,line join=round,line cap=round] (128.41, 76.03) -- (145.29, 76.03);

\path[draw=drawColor,line width= 0.4pt,dash pattern=on 4pt off 4pt ,line join=round,line cap=round] (160.90, 76.03) -- (152.64, 76.03);

\path[draw=drawColor,line width= 0.4pt,line join=round,line cap=round] (128.41, 71.25) -- (128.41, 80.81);

\path[draw=drawColor,line width= 0.4pt,line join=round,line cap=round] (160.90, 71.25) -- (160.90, 80.81);
\definecolor{drawColor}{RGB}{0,0,255}

\path[draw=drawColor,line width= 0.4pt,line join=round,line cap=round] (145.29, 66.46) --
	(145.29, 85.60) --
	(152.64, 85.60) --
	(152.64, 66.46) --
	(145.29, 66.46);
\definecolor{drawColor}{RGB}{255,0,0}

\path[draw=drawColor,line width= 1.2pt,line join=round] (199.97, 90.38) -- (199.97,109.52);
\definecolor{drawColor}{RGB}{0,0,0}

\path[draw=drawColor,line width= 0.4pt,dash pattern=on 4pt off 4pt ,line join=round,line cap=round] ( 98.23, 99.95) -- (192.26, 99.95);

\path[draw=drawColor,line width= 0.4pt,dash pattern=on 4pt off 4pt ,line join=round,line cap=round] (227.12, 99.95) -- (211.68, 99.95);

\path[draw=drawColor,line width= 0.4pt,line join=round,line cap=round] ( 98.23, 95.17) -- ( 98.23,104.74);

\path[draw=drawColor,line width= 0.4pt,line join=round,line cap=round] (227.12, 95.17) -- (227.12,104.74);
\definecolor{drawColor}{RGB}{0,0,255}

\path[draw=drawColor,line width= 0.4pt,line join=round,line cap=round] (192.26, 90.38) --
	(192.26,109.52) --
	(211.68,109.52) --
	(211.68, 90.38) --
	(192.26, 90.38);
\definecolor{drawColor}{RGB}{255,0,0}

\path[draw=drawColor,line width= 1.2pt,line join=round] (363.97,114.30) -- (363.97,133.44);
\definecolor{drawColor}{RGB}{0,0,0}

\path[draw=drawColor,line width= 0.4pt,dash pattern=on 4pt off 4pt ,line join=round,line cap=round] (352.56,123.87) -- (361.26,123.87);

\path[draw=drawColor,line width= 0.4pt,dash pattern=on 4pt off 4pt ,line join=round,line cap=round] (373.96,123.87) -- (367.14,123.87);

\path[draw=drawColor,line width= 0.4pt,line join=round,line cap=round] (352.56,119.09) -- (352.56,128.66);

\path[draw=drawColor,line width= 0.4pt,line join=round,line cap=round] (373.96,119.09) -- (373.96,128.66);
\definecolor{drawColor}{RGB}{0,0,255}

\path[draw=drawColor,line width= 0.4pt,line join=round,line cap=round] (361.26,114.30) --
	(361.26,133.44) --
	(367.14,133.44) --
	(367.14,114.30) --
	(361.26,114.30);
\end{scope}
\begin{scope}
\path[clip] (  0.00,  0.00) rectangle (411.94,187.90);
\definecolor{drawColor}{RGB}{0,0,0}

\path[draw=drawColor,line width= 0.4pt,line join=round,line cap=round] ( 49.20, 76.03) -- ( 49.20,123.87);

\path[draw=drawColor,line width= 0.4pt,line join=round,line cap=round] ( 49.20, 76.03) -- ( 43.20, 76.03);

\path[draw=drawColor,line width= 0.4pt,line join=round,line cap=round] ( 49.20, 99.95) -- ( 43.20, 99.95);

\path[draw=drawColor,line width= 0.4pt,line join=round,line cap=round] ( 49.20,123.87) -- ( 43.20,123.87);

\node[text=drawColor,rotate= 90.00,anchor=base,inner sep=0pt, outer sep=0pt, scale=  1.00] at ( 34.80, 76.03) { \footnotesize{Ours}};

\node[text=drawColor,rotate= 90.00,anchor=base,inner sep=0pt, outer sep=0pt, scale=  1.00] at ( 34.80,123.87) { \footnotesize{Gao}};

\node[text=drawColor,rotate= 90.00,anchor=base,inner sep=0pt, outer sep=0pt, scale=  1.00] at ( 34.80,98.87) { \footnotesize{Usenko}};

\node[text=drawColor,anchor=base,inner sep=0pt, outer sep=0pt, scale=  1.00] at ( 230.80,131.00) { \small{\textbf{Map Generation Time}}};

\path[draw=drawColor,line width= 0.4pt,line join=round,line cap=round] ( 49.20, 61.20) --
	(386.74, 61.20) --
	(386.74,138.70) --
	( 49.20,138.70) --
	( 49.20, 61.20);
\end{scope}
\end{tikzpicture}
      }%  
 
  \end{picture}%
\endgroup%    
 }
%%%%%------ new figure
\subfloat{
    \def\svgwidth{4.9cm}
\begingroup%
  \makeatletter%
  \providecommand\color[2][]{%
    \errmessage{(Inkscape) Color is used for the text in Inkscape, but the package 'color.sty' is not loaded}%
    \renewcommand\color[2][]{}%
  }%
  \providecommand\transparent[1]{%
    \errmessage{(Inkscape) Transparency is used (non-zero) for the text in Inkscape, but the package 'transparent.sty' is not loaded}%
    \renewcommand\transparent[1]{}%
  }%
  \providecommand\rotatebox[2]{#2}%
  \ifx\svgwidth\undefined%
    \setlength{\unitlength}{170.07874016bp}%
    \ifx\svgscale\undefined%
      \relax%
    \else%
      \setlength{\unitlength}{\unitlength * \real{\svgscale}}%
    \fi%
  \else%
    \setlength{\unitlength}{\svgwidth}%
  \fi%
  \global\let\svgwidth\undefined%
  \global\let\svgscale\undefined%
  \makeatother%
  \begin{picture}(1,1)%
      \put(-0.48,-0.558){ % Created by tikzDevice version 0.12 on 2019-04-23 13:54:03
% !TEX encoding = UTF-8 Unicode
\begin{tikzpicture}[x=1pt,y=1pt]
\definecolor{fillColor}{RGB}{255,255,255}
\path[use as bounding box,fill=fillColor,fill opacity=0.00] (0,0) rectangle (411.94,187.90);
\begin{scope}
\path[clip] ( 49.20, 61.20) rectangle (386.74,138.70);
\definecolor{drawColor}{RGB}{255,0,0}

\path[draw=drawColor,line width= 1.2pt,line join=round] (182.64, 66.46) -- (182.64, 85.60);
\definecolor{drawColor}{RGB}{0,0,0}

\path[draw=drawColor,line width= 0.4pt,dash pattern=on 4pt off 4pt ,line join=round,line cap=round] (129.87, 76.03) -- (159.63, 76.03);

\path[draw=drawColor,line width= 0.4pt,dash pattern=on 4pt off 4pt ,line join=round,line cap=round] (226.99, 76.03) -- (205.77, 76.03);

\path[draw=drawColor,line width= 0.4pt,line join=round,line cap=round] (129.87, 71.25) -- (129.87, 80.81);

\path[draw=drawColor,line width= 0.4pt,line join=round,line cap=round] (226.99, 71.25) -- (226.99, 80.81);
\definecolor{drawColor}{RGB}{0,0,255}

\path[draw=drawColor,line width= 0.4pt,line join=round,line cap=round] (159.63, 66.46) --
	(159.63, 85.60) --
	(205.77, 85.60) --
	(205.77, 66.46) --
	(159.63, 66.46);
\definecolor{drawColor}{RGB}{255,0,0}

\path[draw=drawColor,line width= 1.2pt,line join=round] (281.51, 90.38) -- (281.51,109.52);
\definecolor{drawColor}{RGB}{0,0,0}

\path[draw=drawColor,line width= 0.4pt,dash pattern=on 4pt off 4pt ,line join=round,line cap=round] (142.69, 99.95) -- (266.18, 99.95);

\path[draw=drawColor,line width= 0.4pt,dash pattern=on 4pt off 4pt ,line join=round,line cap=round] (309.07, 99.95) -- (291.58, 99.95);

\path[draw=drawColor,line width= 0.4pt,line join=round,line cap=round] (142.69, 95.17) -- (142.69,104.74);

\path[draw=drawColor,line width= 0.4pt,line join=round,line cap=round] (309.07, 95.17) -- (309.07,104.74);
\definecolor{drawColor}{RGB}{0,0,255}

\path[draw=drawColor,line width= 0.4pt,line join=round,line cap=round] (266.18, 90.38) --
	(266.18,109.52) --
	(291.58,109.52) --
	(291.58, 90.38) --
	(266.18, 90.38);
\definecolor{drawColor}{RGB}{255,0,0}

\path[draw=drawColor,line width= 1.2pt,line join=round] (308.87,114.30) -- (308.87,133.44);
\definecolor{drawColor}{RGB}{0,0,0}

\path[draw=drawColor,line width= 0.4pt,dash pattern=on 4pt off 4pt ,line join=round,line cap=round] (249.22,123.87) -- (302.05,123.87);

\path[draw=drawColor,line width= 0.4pt,dash pattern=on 4pt off 4pt ,line join=round,line cap=round] (333.13,123.87) -- (318.85,123.87);

\path[draw=drawColor,line width= 0.4pt,line join=round,line cap=round] (249.22,119.09) -- (249.22,128.66);

\path[draw=drawColor,line width= 0.4pt,line join=round,line cap=round] (333.13,119.09) -- (333.13,128.66);
\definecolor{drawColor}{RGB}{0,0,255}

\path[draw=drawColor,line width= 0.4pt,line join=round,line cap=round] (302.05,114.30) --
	(302.05,133.44) --
	(318.85,133.44) --
	(318.85,114.30) --
	(302.05,114.30);
\end{scope}
\begin{scope}
\path[clip] (  0.00,  0.00) rectangle (411.94,187.90);
\definecolor{drawColor}{RGB}{0,0,0}

\path[draw=drawColor,line width= 0.4pt,line join=round,line cap=round] ( 49.20, 76.03) -- ( 49.20,123.87);

\path[draw=drawColor,line width= 0.4pt,line join=round,line cap=round] ( 49.20, 76.03) -- ( 43.20, 76.03);

\path[draw=drawColor,line width= 0.4pt,line join=round,line cap=round] ( 49.20, 99.95) -- ( 43.20, 99.95);

\path[draw=drawColor,line width= 0.4pt,line join=round,line cap=round] ( 49.20,123.87) -- ( 43.20,123.87);

\node[text=drawColor,rotate= 90.00,anchor=base,inner sep=0pt, outer sep=0pt, scale=  1.00] at ( 34.80, 76.03) { \footnotesize{Ours}};

\node[text=drawColor,rotate= 90.00,anchor=base,inner sep=0pt, outer sep=0pt, scale=  1.00] at ( 34.80,123.87) { \footnotesize{Gao}};

\node[text=drawColor,rotate= 90.00,anchor=base,inner sep=0pt, outer sep=0pt, scale=  1.00] at ( 34.80,98.87) { \footnotesize{Usenko}};

\node[text=drawColor,anchor=base,inner sep=0pt, outer sep=0pt, scale=  1.00] at ( 229.80,131.00) { \small{\textbf{Motion Generation Time}}};

\path[draw=drawColor,line width= 0.4pt,line join=round,line cap=round] ( 61.70, 61.20) -- (311.73, 61.20);

\path[draw=drawColor,line width= 0.4pt,line join=round,line cap=round] ( 61.70, 61.20) -- ( 61.70, 55.20);

\path[draw=drawColor,line width= 0.4pt,line join=round,line cap=round] (186.72, 61.20) -- (186.72, 55.20);

\path[draw=drawColor,line width= 0.4pt,line join=round,line cap=round] (311.73, 61.20) -- (311.73, 55.20);

\node[text=drawColor,anchor=base,inner sep=0pt, outer sep=0pt, scale=  1.00] at ( 61.70, 45.00) { \small{1e-02}};

\node[text=drawColor,anchor=base,inner sep=0pt, outer sep=0pt, scale=  1.00] at (186.72, 45.00) { \small{1e+00}};

\node[text=drawColor,anchor=base,inner sep=0pt, outer sep=0pt, scale=  1.00] at (311.73, 45.00) { \small{1e+02}};

\node[text=drawColor,anchor=base,inner sep=0pt, outer sep=0pt, scale=  1.00] at (234.72, 37.00) { miliseconds$[\mathrm{ms}]$};

\path[draw=drawColor,line width= 0.4pt,line join=round,line cap=round] ( 49.20, 61.20) --
	(386.74, 61.20) --
	(386.74,138.70) --
	( 49.20,138.70) --
	( 49.20, 61.20);
\end{scope}
\end{tikzpicture} }%  
 
  \end{picture}%
\endgroup%    
}
%%%%%------ new figure
\subfloat{\label{subfig:comp3}
    \def\svgwidth{4.9cm}
\begingroup%
  \makeatletter%
  \providecommand\color[2][]{%
    \errmessage{(Inkscape) Color is used for the text in Inkscape, but the package 'color.sty' is not loaded}%
    \renewcommand\color[2][]{}%
  }%
  \providecommand\transparent[1]{%
    \errmessage{(Inkscape) Transparency is used (non-zero) for the text in Inkscape, but the package 'transparent.sty' is not loaded}%
    \renewcommand\transparent[1]{}%
  }%
  \providecommand\rotatebox[2]{#2}%
  \ifx\svgwidth\undefined%
    \setlength{\unitlength}{170.07874016bp}%
    \ifx\svgscale\undefined%
      \relax%
    \else%
      \setlength{\unitlength}{\unitlength * \real{\svgscale}}%
    \fi%
  \else%
    \setlength{\unitlength}{\svgwidth}%
  \fi%
  \global\let\svgwidth\undefined%
  \global\let\svgscale\undefined%
  \makeatother%
  \begin{picture}(0,0)%
      \put(-2.45, -0.06){ \includegraphics[width=\unitlength]{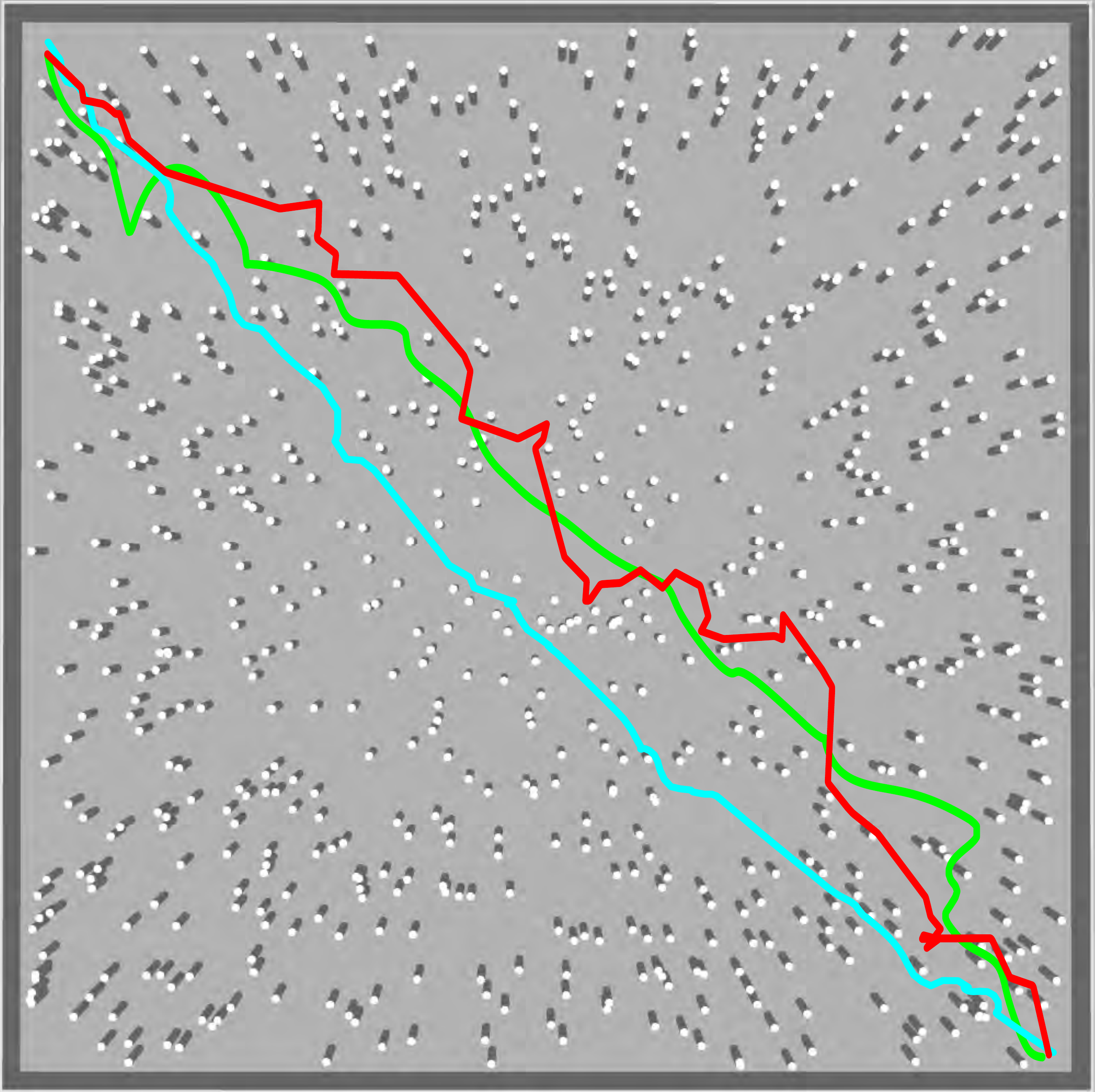} }%  
 
  \end{picture}%
  
\endgroup%    
}\vspace{0.7in}

    \caption{Qualtative comparision against Gao et. al.~\cite{Gao2018OnlinePolynomial} and Usenko et. al.~\cite{Usenko2017Real-TimeBuffer}. \textbf{On the left.} Top view of the simulated forest environment, together with the traversed paths. Cyan-Usenko et. al.~\cite{Usenko2017Real-TimeBuffer}; Green-Gao et. al.~\cite{Gao2018OnlinePolynomial} and Red-Ours.  \textbf{On the right.} Boxplot, on a logarithmic scale, of the computed time of each algorithm for map generation and motion generation.}
    \label{fig:comparisons}
\end{figure*}
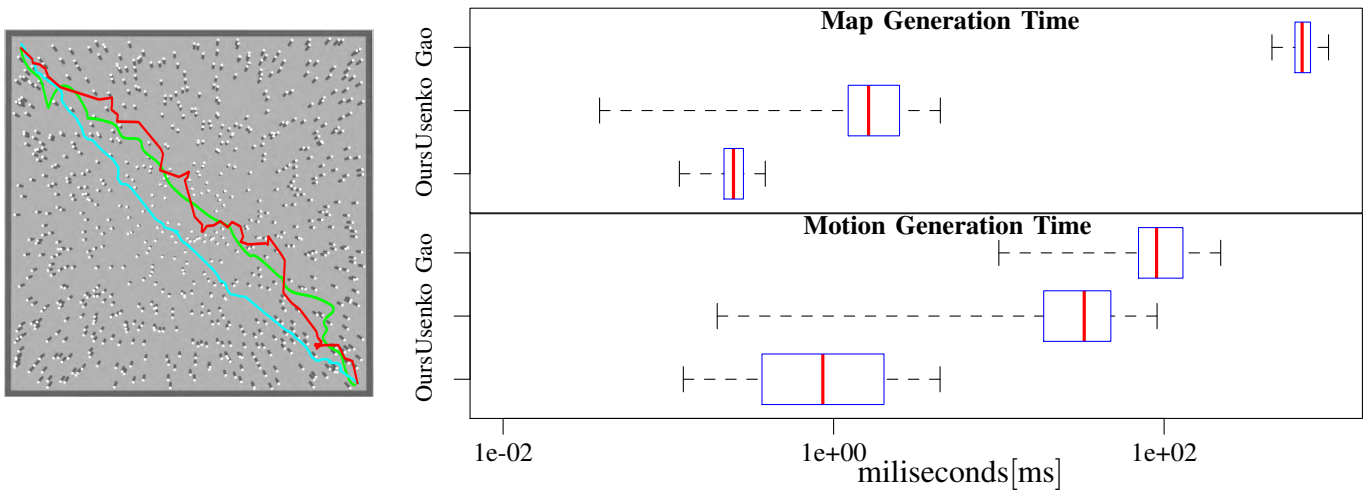

\subsection{Benchmark}
 The benchmark experiments were performed on an Intel Core i7-5557U @ 3.1 GHz, 8 GB of RAM, Ubuntu 16.04-LTS operating system, and using ROS Kinetic. The algorithm was compared against two state-of-the-art navigation algorithms~\cite{Gao2018OnlinePolynomial}~\cite{Usenko2017Real-TimeBuffer}. The simulated experiments were performed in several Poisson forest scenarios with a tree density of $0.3$ $trees/m^2$ in a space of $50\text{m}\times50\text{m}\times2\text{m}$. The trees were limited to a height of two meters and a radius of twenty centimeters. The start and goal positions of the simulated robot were chosen in opposite corners to maximize the distance traveled. The trees per cubic meter metric represent $ 26.6 \%$ of the volume. The environments were simulated using Gazebo together with Rotor-S MAV~\cite{Furrer2016RotorSAFramework} and RealSense simulation packages. The simulated robot was selected as the Hummingbird and its sensor a RealSense depth camera with 70 degrees of horizontal FoV and 43 degrees of vertical FoV. Two sets of metrics were analyzed in the benchmark. The first set relates to the computation time of each algorithm: map generation time and motion generation time. The second set describes the overall performance of the generated trajectories in each algorithm: total trajectory length, average velocity and time to reach the target goal.

Some parameters had to be tuned to obtain a high success rate in our simulated forest environments, starting from the configuration reported in each paper.  The method as reported in~\cite{Gao2018OnlinePolynomial} assumes a 360-degrees horizontal FoV and ten meters sensing range. For this reason, in the experiments for~\cite{Gao2018OnlinePolynomial}, the observable space was set to cover a front horizontal FoV of 180-degrees by using a multirotor MAV with three RealSense cameras with a depth limit of $7 \text{m}$. The maximum velocity of the robot was set to $0.5 \text{m}/\text{s}$, the maximum acceleration to $0.5 \text{m}/\text{s}^2$ and the check horizon and stop horizon parameters were set to $8.5 \text{m}$ and $1.0 \text{m}$, respectively.

The implementation of the algorithm in~\cite{Usenko2017Real-TimeBuffer} presented some limitations as we were not able to achieve a perfect success rate even after careful tuning. The main reason is that dynamic constraints in yaw and acceleration are not taken into account in the trajectory generation leading to two types of failures. On one hand collisions occur, either because the robot is moving too fast and is unable to stop before a detected obstacle, or because it changes direction too quickly and is unable to align the FoV with the velocity before running into unobserved obstacles. On the other hand, in some cases the robot became unstable when the controller tried to follow trajectories requiring large accelerations. The best success rate in our experiments was obtained by setting a maximum velocity of $0.7\text{m}/\text{s}$ and a maximum acceleration of $0.7\text{m}/\text{s}^2$ for the robot. The distance threshold, number of optimization points and the dt parameters of the algorithm were set to $0.5\text{m}$, $5$ and $0.5\text{s}$, respectively.

In our algorithm, we configured the maximum depth with three meters, the minimum voxel radius of 0.1 meters and the maximum virtual map size of 150 meters.  The parameters used for the trajectory generation were $E_p=0.1\text{m}$, $E_\yaw=1\text{rad}$, for the maximum separation in the position and the orientation, respectively. For the maximum velocity $V_{\max}=1\text{m}/\text{s}$, the maximum acceleration $A_{\max}=1.0\text{m}/\text{s}^2$, the maximum jerk $J_{\max}=1.0\text{m}/\text{s}^3$, and the maximum angular velocity $\Omega_{\max}=0.2\text{rad}/s$. Moreover, we use the following gains: $k_1=~55.0$, $k_2=-843.75$, $k_3=-5406.2$, $k_4=-10687.5$, $k_5=-10.5$, $k_6=-33.3$.

The results of this benchmark are presented in Figure~\ref{fig:comparisons} and Table~\ref{Tab:results}. Note that our algorithm outperforms~\cite{Gao2018OnlinePolynomial} and~\cite{Usenko2017Real-TimeBuffer} in the first set of metrics shown in Figure~\ref{fig:comparisons}, with an average motion time of $3.37 \text{ms}$ against $103.2 \text{ms}$ obtained by~\cite{Gao2018OnlinePolynomial}  and $35.5 \text{ms}$ obtained by~\cite{Usenko2017Real-TimeBuffer}; and an average mapping time of $0.256\text{ms}$ against $700.7 \text{ms}$ obtained by~\cite{Gao2018OnlinePolynomial}  and $2.035 \text{ms}$ obtained by~\cite{Usenko2017Real-TimeBuffer}. Regarding the benchmark of the performance of the trajectories, the results are presented in Table~\ref{Tab:results} where it can be noted that our method outperforms ~\cite{Gao2018OnlinePolynomial} and~\cite{Usenko2017Real-TimeBuffer} in average velocity of the robot and average navigation time. However, our approach generated, on average, slightly longer paths.

\begin{table}
\centering

\begin{tabular}{c|ccc}
\textbf{Algorithm} & \textbf{\begin{tabular}[c]{@{}c@{}}Avg. Path\\ Length {[}m{]}\end{tabular}} & \textbf{\begin{tabular}[c]{@{}c@{}}Avg. \\ Velocity {[}m/s{]}\end{tabular}} & \textbf{\begin{tabular}[c]{@{}c@{}}Avg. Nav.\\ Time {[}s{]}\end{tabular}} \\ \hline
Gao et. al. \cite{Gao2018OnlinePolynomial}  &   82.56   &  0.2182    &  406.41 \\
Usenko et. al. \footnote{Results for this method are shown although not all tests were completed i.e., the robot crashed with the obstacles since several generated trajectories were dynamically unfeasible} \cite{Usenko2017Real-TimeBuffer} & \textbf{77.572} & 0.48937 & 170.02\\
Ours  &   92.752  &    \textbf{0.79753}    &      \textbf{146.30 }                             
\end{tabular}
\caption{The overall performance of the generated trajectories}
\label{Tab:results}
\end{table}

\begin{figure*}
    \centering
    \subfloat[]{\label{subfig:exp1}

\def\svgwidth{4.20cm}
\begingroup%
  \makeatletter%
  \providecommand\color[2][]{%
    \errmessage{(Inkscape) Color is used for the text in Inkscape, but the package 'color.sty' is not loaded}%
    \renewcommand\color[2][]{}%
  }%
  \providecommand\transparent[1]{%
    \errmessage{(Inkscape) Transparency is used (non-zero) for the text in Inkscape, but the package 'transparent.sty' is not loaded}%
    \renewcommand\transparent[1]{}%
  }%
  \providecommand\rotatebox[2]{#2}%
  \ifx\svgwidth\undefined%
    \setlength{\unitlength}{170.07874016bp}%
    \ifx\svgscale\undefined%
      \relax%
    \else%
      \setlength{\unitlength}{\unitlength * \real{\svgscale}}%
    \fi%
  \else%
    \setlength{\unitlength}{\svgwidth}%
  \fi%
  \global\let\svgwidth\undefined%
  \global\let\svgscale\undefined%
  \makeatother%
  \begin{picture}(1,1)%  
      \put(0,0){\includegraphics[width=\unitlength]{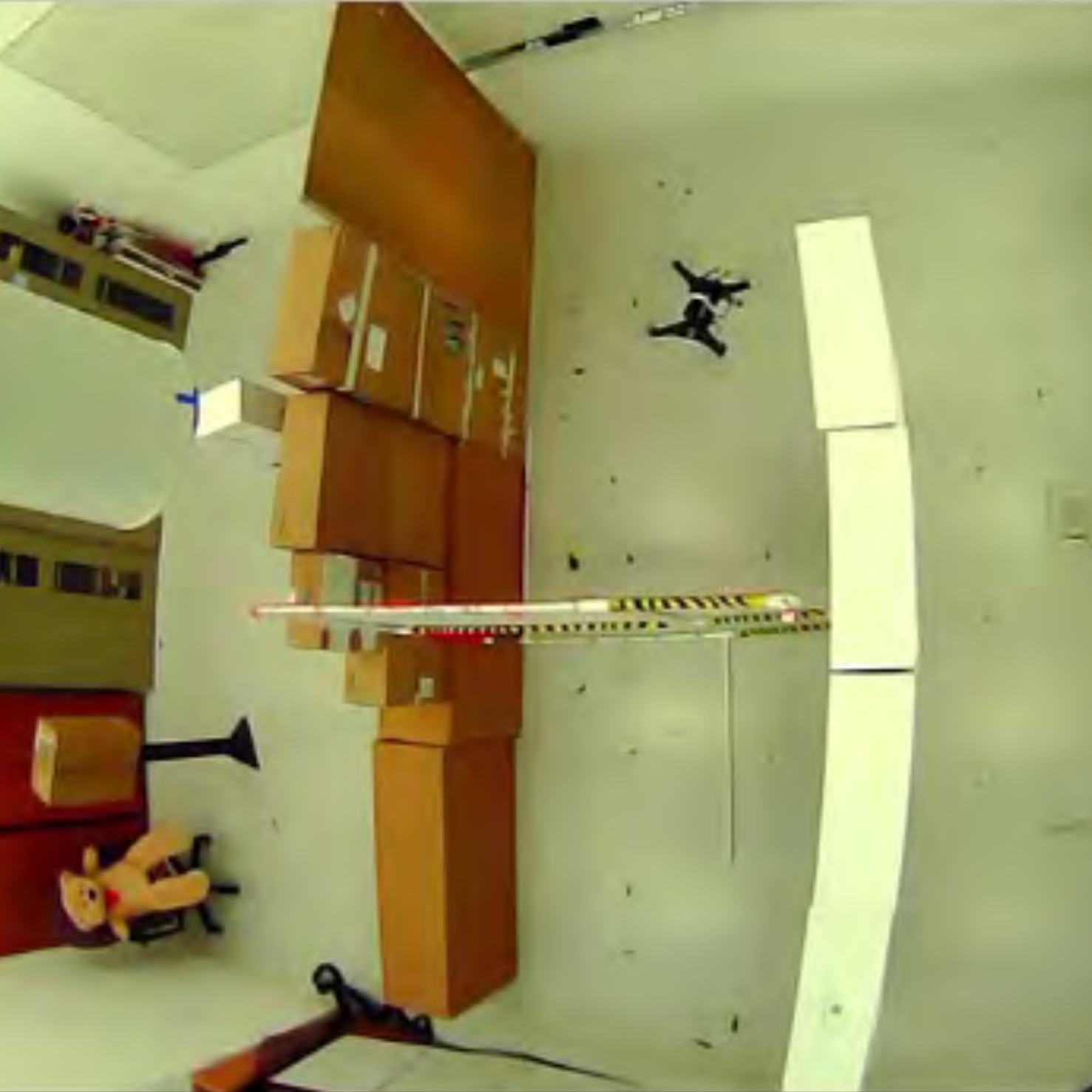}}%
    \end{picture}%
\endgroup%    
    }
    \hspace{-0.38cm}
    \subfloat[]{\label{subfig:exp2}
    \def\svgwidth{4.20cm}
\begingroup%
  \makeatletter%
  \providecommand\color[2][]{%
    \errmessage{(Inkscape) Color is used for the text in Inkscape, but the package 'color.sty' is not loaded}%
    \renewcommand\color[2][]{}%
  }%
  \providecommand\transparent[1]{%
    \errmessage{(Inkscape) Transparency is used (non-zero) for the text in Inkscape, but the package 'transparent.sty' is not loaded}%
    \renewcommand\transparent[1]{}%
  }%
  \providecommand\rotatebox[2]{#2}%
  \ifx\svgwidth\undefined%
    \setlength{\unitlength}{170.07874016bp}%
    \ifx\svgscale\undefined%
      \relax%
    \else%
      \setlength{\unitlength}{\unitlength * \real{\svgscale}}%
    \fi%
  \else%
    \setlength{\unitlength}{\svgwidth}%
  \fi%
  \global\let\svgwidth\undefined%
  \global\let\svgscale\undefined%
  \makeatother%
  \begin{picture}(1,1)%
      \put(0,0){\includegraphics[width=\unitlength,page=1]{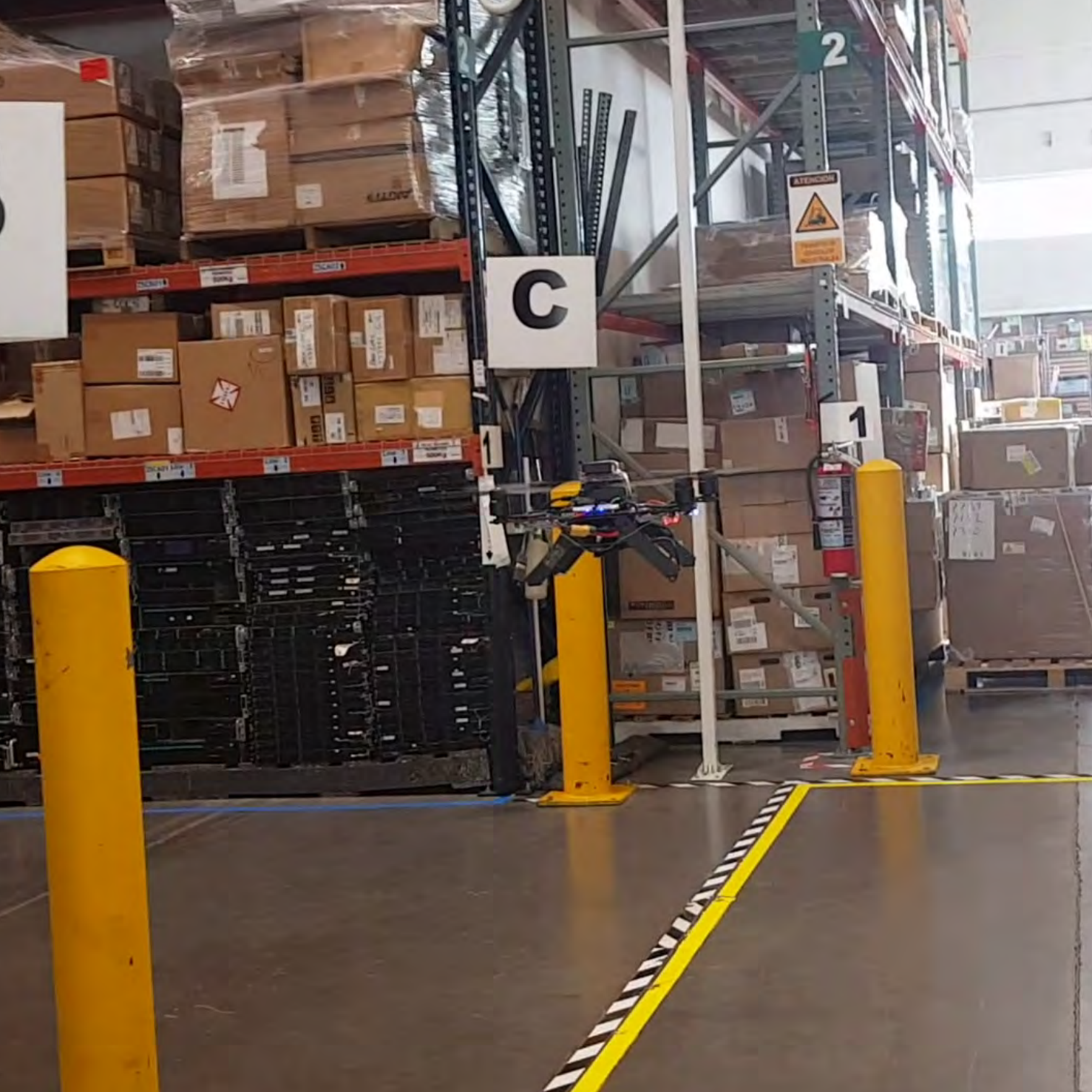}}%    

  \end{picture}%
\endgroup%    
    }\hspace{-0.38cm}
    %\vspace{-0.35cm}
    \subfloat[]{\label{subfig:exp3}
    \def\svgwidth{4.20cm}
\begingroup%
  \makeatletter%
  \providecommand\color[2][]{%
    \errmessage{(Inkscape) Color is used for the text in Inkscape, but the package 'color.sty' is not loaded}%
    \renewcommand\color[2][]{}%
  }%
  \providecommand\transparent[1]{%
    \errmessage{(Inkscape) Transparency is used (non-zero) for the text in Inkscape, but the package 'transparent.sty' is not loaded}%
    \renewcommand\transparent[1]{}%
  }%
  \providecommand\rotatebox[2]{#2}%
  \ifx\svgwidth\undefined%
    \setlength{\unitlength}{170.07874016bp}%
    \ifx\svgscale\undefined%
      \relax%
    \else%
      \setlength{\unitlength}{\unitlength * \real{\svgscale}}%
    \fi%
  \else%
    \setlength{\unitlength}{\svgwidth}%
  \fi%
  \global\let\svgwidth\undefined%
  \global\let\svgscale\undefined%
  \makeatother%
  \begin{picture}(1,1)%
      \put(0,0){\includegraphics[width=\unitlength,page=1]{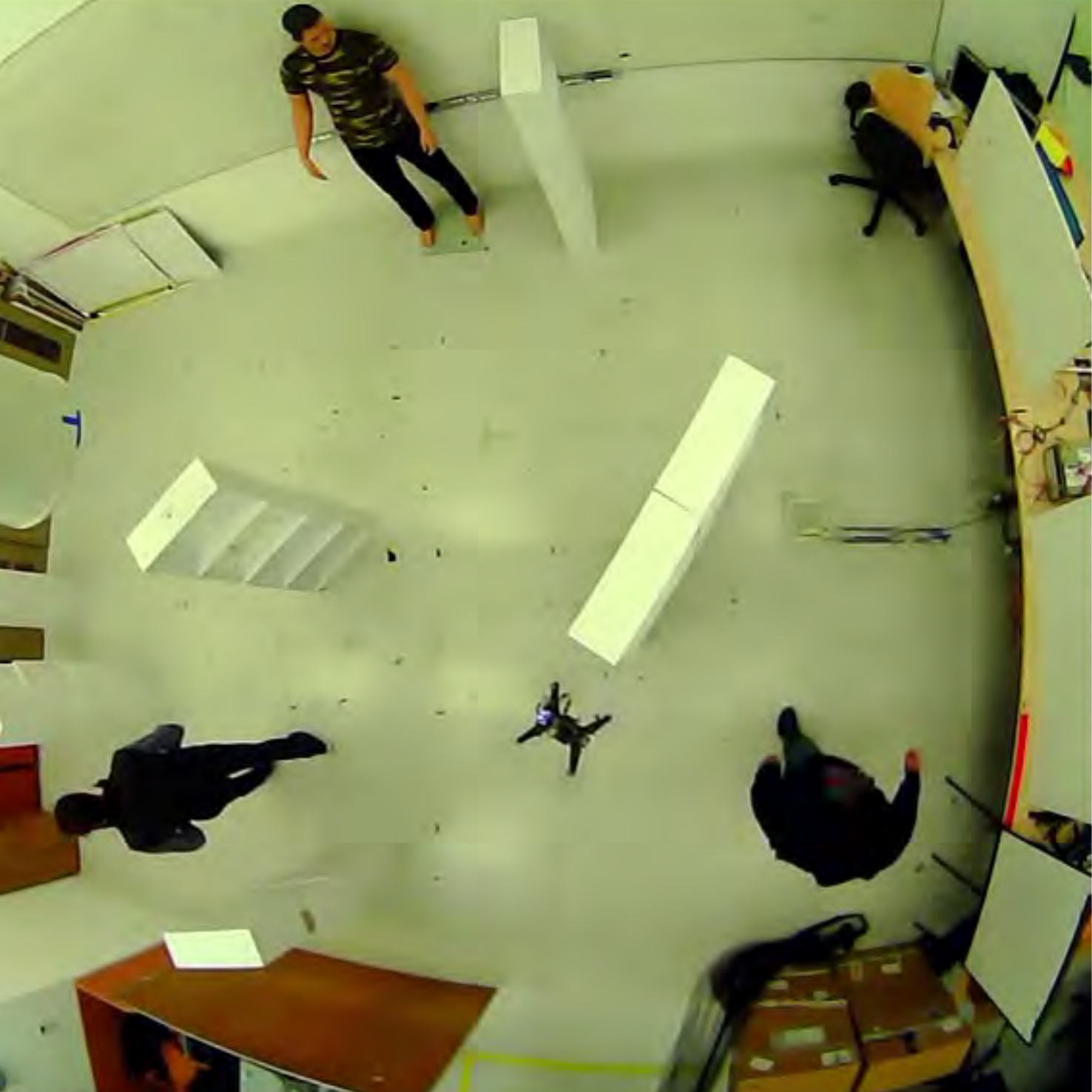}}
    
  \end{picture}%
\endgroup%    
    }\hspace{-0.38cm}
   % \vspace{-0.35cm}
    \subfloat[]{\label{subfig:exp4}
    \def\svgwidth{4.20cm}
\begingroup%
  \makeatletter%
  \providecommand\color[2][]{%
    \errmessage{(Inkscape) Color is used for the text in Inkscape, but the package 'color.sty' is not loaded}%
    \renewcommand\color[2][]{}%
  }%
  \providecommand\transparent[1]{%
    \errmessage{(Inkscape) Transparency is used (non-zero) for the text in Inkscape, but the package 'transparent.sty' is not loaded}%
    \renewcommand\transparent[1]{}%
  }%
  \providecommand\rotatebox[2]{#2}%
  \ifx\svgwidth\undefined%
    \setlength{\unitlength}{170.07874016bp}%
    \ifx\svgscale\undefined%
      \relax%
    \else%
      \setlength{\unitlength}{\unitlength * \real{\svgscale}}%
    \fi%
  \else%
    \setlength{\unitlength}{\svgwidth}%
  \fi%
  \global\let\svgwidth\undefined%
  \global\let\svgscale\undefined%
  \makeatother%
  \begin{picture}(1,1)%
      \put(0,0){\includegraphics[width=\unitlength]{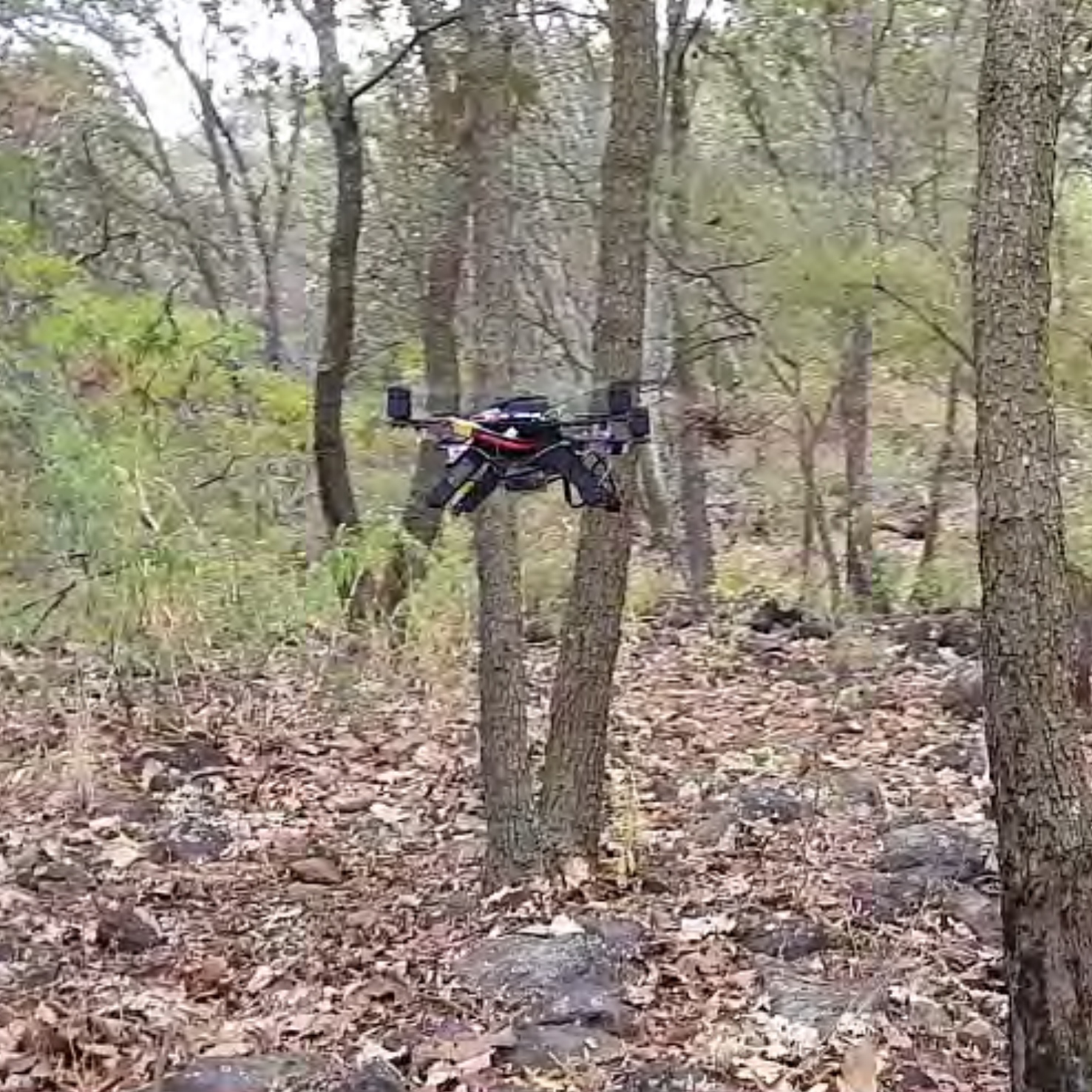}}%  
   
  \end{picture}%
\endgroup%    
    
    }
    
    \caption{Environments for real-world experiments: (a) 3D maze scenario. (b) Warehouse scenario. (c) Dynamic environment. (d) Forest environment.}
    \label{fig:experiments}
\end{figure*}

A second scenario in which the algorithms were tested is in the maze environment shown in Figure~\ref{fig:maze_results}. This scenario is challenging as a successful completion requires the ability to escape ``pockets" and to maneuver in very tight spaces. Hence, navigating inside the field of view becomes of paramount importance. A typical path execution using our algorithm is illustrated in Figure~\ref{fig:maze_results}. Notice the loops in the trajectory where the MAV encountered pockets. Rather than stopping completely and then turning around while hovering in place, the MAV was able to turn around while continuing to make forward progress safely by navigating inside the field of view all the time. We did our best to complete this scenario with the algorithms ~\cite{Usenko2017Real-TimeBuffer,Gao2018OnlinePolynomial} but unfortunately none of our attempts were successful. In the case of~\cite{Usenko2017Real-TimeBuffer} we believe failure to account for yaw dynamical constraints of the system in the planning algorithm was the main cause for the navigation always ending in collisions with the maze walls. In the case of~\cite{Gao2018OnlinePolynomial}, we extended the setup to six cameras covering the full 360-degrees of horizontal field of view and we carefully tuned the algorithm parameters including the polynomial order in the optimizer. Nonetheless, the maze environment could not be completed as in every experiment there was a point where the optimizer was unable to find feasible solutions even after many re-plannings. This environment demonstrates the advantage of the trajectory generation algorithm here proposed. As the results indicate, compared to polynomial trajectories, our method has the advantage of navigating through tight corners, such as doorways in an indoor environment.

\begin{figure}
    \centering
    \includegraphics[width=3.4cm]{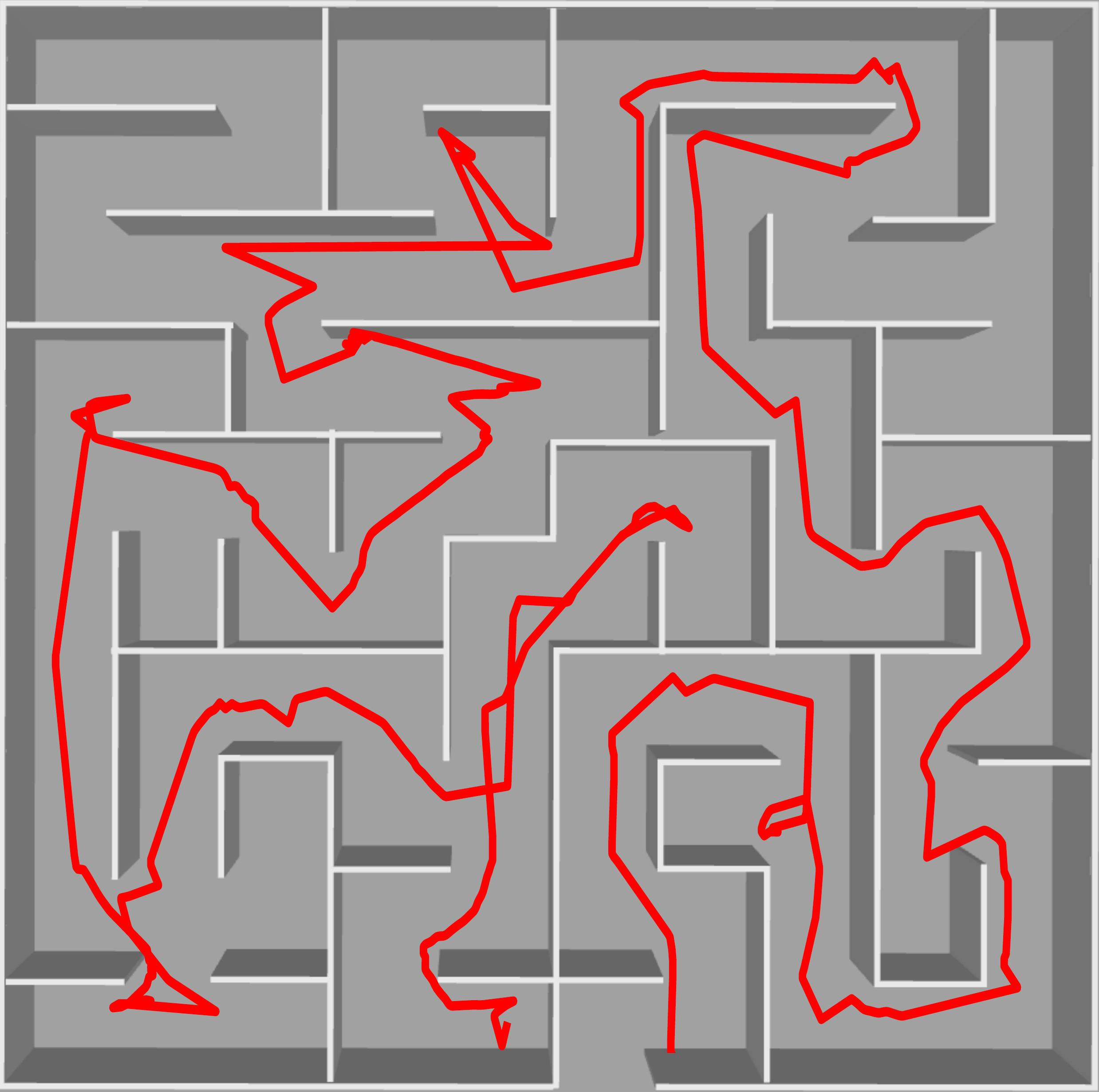}
    \includegraphics[width=5cm]{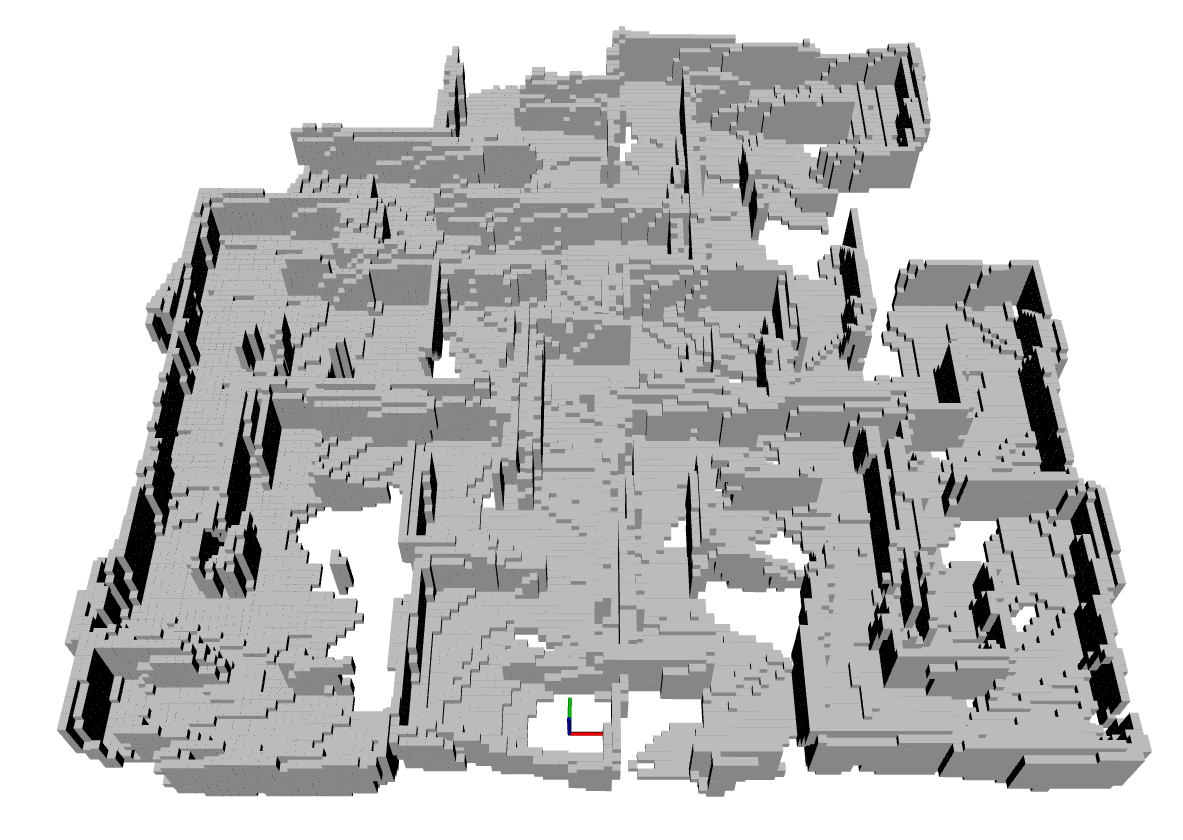}
    \caption{Navigation in an unknown maze environment. \textbf{On the left.} In red, the generated trajectory of our algorithm to reach the goal. The start and goal positions are given in blue and red, respectively. \textbf{On the right.} The map created while finding the goal.}
    \label{fig:maze_results}
\end{figure}

\subsection{Real-World Experiments}

To illustrate the effectiveness of our approach for navigating in real unknown scenarios, experiments were performed in the different static and dynamic environments illustrated in Figure~\ref{fig:experiments}. The framework was implemented on an Intel Aero Ready-To-Fly drone\footnote{\url{https://www.intel.com/content/www/us/en/products/drones/aero-ready-to-fly.html}} with an Intel Atom x7-Z8750 @ 1.6 GHz, 4 Gb of RAM, Ubuntu 16.04-LTS operating system with ROS Kinetic, an Intel~RealSense~Tracking~Camera~T265~\footnote{\url{https://www.intelrealsense.com/tracking-camera-t265/}} for visual-inertial odometry and an Intel~RealSense~Depth~Camera~D435\footnote{\url{https://www.intelrealsense.com/depth-camera-d435/}} for depth images sensing. 

%Although it was preferred the Intel~RealSense~Tracking~Camera~T265, other visual odometries are also suitable for our implementation, such as~\cite{Qin2018VINS-Mono:Estimator}~\cite{Forster2016SVO:Systems}. 

The entire framework runs completely on-board at 33Hz alongside with a custom nonlinear control algorithm for precise trajectory tracking. To the best of our knowledge, these experiments represent the implementation in the most frugal platform for navigation in unknown cluttered environments demonstrated to date.  A video of the different tests can be found in ~\url{https://youtu.be/Wq0e7vF6nZM}.

We configured the maximum depth with four meters, the minimum voxel radius of 0.06 meters and the maximum virtual map size of 20 meters. The parameters used for the trajectory generation were $E_p=0.1\text{m}$, $E_\yaw=1\text{rad}$, for the maximum separation in the position and the orientation, respectively. For the maximum velocity $V_{\max}=1\text{m}/\text{s}$, the maximum acceleration $A_{\max}=0.1\text{m}/\text{s}^2$, the maximum jerk $J_{\max}=0.1\text{m}/\text{s}^3$, and the maximum angular velocity $\Omega_{\max}=0.2\text{rad}/s$. Moreover, we use the following gains: $k_1=~35.0$, $k_2=443.75$, $k_3=-2406.2$, $k_4=-4687.5$, $k_5=-10.5$, $k_6=-33.3$. We obtained the average computation time of each algorithm while running in the Intel Atom microprocessor: map generation time with $3.07\text{ms}$  and motion generation time with $5.1\text{ms}$ seconds.

The first scenario, shown in Figure~\ref{subfig:exp1}, is a 3D maze, where the quadcopter had to navigate from it's initial position in one corner of the maze area to the opposite corner. This scenario is interesting since the robot had to navigate through narrow passages, around tight blind corners and through windows at different heights. The second scenario, shown in Figure~\ref{subfig:exp2}, is an industrial warehouse. This scenario was chosen because it presents a potential real-world application for autonomous MAVs, e.g., for capturing inventory and locating out of place items. The third scenario, shown in Figure~\ref{subfig:exp3}, is a dynamic cluttered environment with people walking inside a lab area. This experiment exhibits the importance of keeping an up-to-date occupied space since, without this feature, the dynamic obstacles would create ``virtual" walls blocking the paths to the goal. It is important to highlight that, from the methods described in Table~\ref{tab:comparison} only our framework has been demonstrated in a dynamic environment. Finally, the fourth scenario, shown in and Figure~\ref{fig:mainfig} and Figure~\ref{subfig:exp4}, is an outdoors exploration in a forest environment, which was chosen because represents a cluttered unstructured and natural scenario with irregular obstacles. 

\section{Conclusion}
We have proposed a complete framework for autonomous navigation. In more detail, the framework allows robots to plan and execute trajectories in cluttered environments, simultaneously generating a map that can be used by the robot to escape pockets and reach target locations in complex 3D environments. Our analytical and numerical results show that the trajectories generated by the framework are safe because they take into account both field of view restrictions as well as dynamical constraints of the robot. The framework was implemented in a commercial multirotor MAV kit and was demonstrated to work real-time, at the camera frame rate of 30Hz, without a discrete GPU or a high-end CPU. Future extensions of this work include the integration of trajectory tracking and prediction of dynamic obstacles, which would allow the robot to navigate more effectively in crowded environments. 

%\bibliographystyle{IEEEtran}
%\bibliography{biblio}

\end{document}